\documentclass[10pt,journal,compsoc]{IEEEtran}

\newcommand{\arxiv}{}

\ifCLASSOPTIONcompsoc
  \usepackage[nocompress]{cite}
\else
  \usepackage{cite}
\fi

\ifCLASSOPTIONcompsoc
  \usepackage[caption=false,font=footnotesize,labelfont=sf,textfont=sf]{subfig}
\else
  \usepackage[caption=false,font=footnotesize]{subfig}
\fi
\newfont{\mycrnotice}{ptmr8t at 7pt}
\newfont{\myconfname}{ptmri8t at 7pt}

\usepackage{amsmath}
\usepackage{amssymb}
\usepackage{amsthm}
\usepackage{color}
\usepackage{pgfplots}
\usepackage{tabularx}

\ifdefined\arxiv
\usepackage[final]{changes}
\else
\usepackage{changes}
\fi

\usepackage{amsmath}
\usepackage{amssymb}
\usepackage{latexsym}
\usepackage{array}
\usepackage{stmaryrd}
\usepackage{paralist}
\usepackage{graphicx} 
\usepackage{caption}
\usepackage{xspace}
\usepackage{subfig}
\usepackage{url}
\usepackage{tabulary,booktabs,multirow}
\usepackage[ruled,lined,linesnumbered,boxed,vlined]{algorithm2e}
\usepackage{hyperref}

\usepackage{tikz}
\usetikzlibrary{positioning,arrows,automata,shapes,decorations.pathreplacing}
\usepackage{footnote}
\usepackage{subfig}
\usepackage{xcolor}

\newcommand{\set}[1]{\ensuremath{\left\{#1\right\}}} 

\newcommand{\cO}{\mathcal{O}}
\newcommand{\cA}{\mathcal{A}}

\newtheorem{theorem}{Theorem}
\newtheorem{proposition}{Proposition}
\newtheorem{corollary}{Corollary}

\newtheorem{remark}{Remark}
\newtheorem{lemma}{Lemma}

\title{Solving the Workflow Satisfiability Problem using General Purpose Solvers\thanks{This is an extended version of \cite{GutinK20} published in the proceedings of the 25th ACM Symposium on Access Control Models and Technologies.}}

\begin{document}

\IEEEtitleabstractindextext{%
\begin{abstract}
The workflow satisfiability problem (WSP) is a well-studied problem in access control seeking allocation of authorised users to every step of the workflow, subject to workflow specification constraints.
It was noticed that the number $k$ of steps is typically small compared to the number of users in the real-world instances of WSP; therefore $k$ is considered as the parameter in WSP parametrised complexity research.
While WSP in general was shown to be W[1]-hard, WSP restricted to a special case of user-independent (UI) constraints is fixed-parameter tractable (FPT)\@.
However, restriction to the UI constraints might be impractical.

To efficiently handle non-UI constraints, we introduce the notion of branching factor of a constraint.
As long as the branching factors of the constraints are relatively small and the number of non-UI constraints is reasonable, WSP can be solved in FPT time.

Extending the results from Karapetyan et al.~(2019), we demonstrate that general-purpose solvers are capable of achieving FPT-like performance on WSP with arbitrary constraints when used with appropriate formulations.
This enables one to tackle most of practical WSP instances.
While important on its own, we hope that this result will also motivate researchers to look for FPT-aware formulations of other FPT problems.
\end{abstract}

}



\author{Daniel~Karapetyan,
        and~Gregory~Gutin
\IEEEcompsocitemizethanks{\IEEEcompsocthanksitem D. Karapetyan was with the School of Computer Science, University of Nottingham, UK.\protect\\
E-mail: daniel.karapetyan@nottingham.ac.uk
\IEEEcompsocthanksitem G. Gutin was with the Department of Computer Science, Royal Holloway, University of London, UK.\protect\\
E-mail: g.gutin@rhul.ac.uk
}}

\maketitle%
\IEEEdisplaynontitleabstractindextext

\section{Introduction}
\label{sec:intro}

Many businesses and other organisations use computerised systems to manage their business processes.
A common example of such a system is a workflow management system, which is responsible for the coordination and execution of steps in a business process.
Such a system is normally multi-user and thus it should include some form of access control which is facilitated by various restrictions on users to perform
steps. 
It can be highly non-trivial to decide whether all the steps can be assigned to available users such that all restrictions are satisfied. 
Such a decision problem
is called the \textsc{Workflow Satisfiability Problem} (WSP).

\textbf{Example:} 
Let us consider the following simple, illustrative example of an instance of the WSP\@.
Figure~\ref{fig:workflow-example} depicts a purchase order processing introduced in~\cite{Cr05}.
As shown in Figure~\ref{fig:workflow-example-steps},
in the first two steps of the workflow, the purchase order is created and approved (and then dispatched to the supplier).
The supplier will submit an invoice for the goods ordered, which is processed by the `create payment' step.
When the supplier delivers the goods, a goods received note (GRN) must be signed and countersigned.
Only then may the payment be approved and sent to the supplier. 

\begin{figure}[h]
\centering
\footnotesize
\begin{minipage}{.45\columnwidth}
\centering
\subfloat[Tasks]{
\label{fig:workflow-example-steps}
\setlength{\extrarowheight}{1pt}
  \begin{tabular}{ll}
    \hline
    $s_1$ & create purchase order \\
    $s_2$ & approve purchase order \\
    $s_3$ & sign GRN \\
    $s_4$ & create payment \\
    $s_5$ & countersign GRN \\
    $s_6$ & approve payment \\
    \hline
  \end{tabular}
}
\end{minipage}
\hfill
\begin{minipage}{.45\columnwidth}
\centering
\subfloat[Constraints]{
\label{fig:workflow-example-constraints}
\centering
\begin{tikzpicture}[t/.style={circle,draw,inner sep=2pt,minimum width=12pt}]
  \node[t] (t1) {$s_1$};
  \node[t,above=of t1] (t2)  {$s_2$};
  \node[t,left=of t1] (t3)  {$s_3$};
  \node[t,right=of t1] (t4)  {$s_4$};
  \node[t,above=of t3] (t5)  {$s_5$};
  \node[t,above=of t4] (t6)  {$s_6$};
  \path (t1) edge [dotted] node {$=$} (t3)
        (t3) edge [dotted] node {$\ne$} (t5)
        (t1) edge [dotted] node[swap] {$\ne$} (t4)
        (t1) edge [dotted] node[swap] {$\ne$} (t2)
        (t4) edge [dotted] node[swap] {$\ne$} (t6);
\end{tikzpicture}
}
\end{minipage}
\caption{A simple constrained workflow for purchase order processing}\label{fig:workflow-example}
\end{figure}
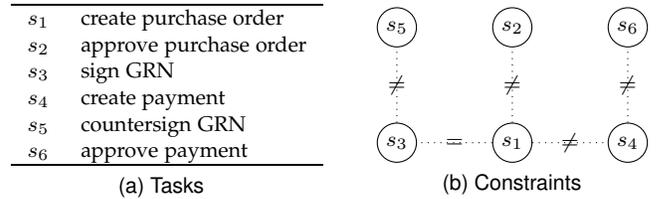

Figure~\ref{fig:workflow-example-constraints} shows constraints to prevent possible fraudulent use of the purchase order processing system.
In our example, these constraints restrict the users that can perform pairs of steps in the workflow: the same user cannot sign and countersign the GRN, for example.
There may also be a requirement that some constraints are performed by the same user.
In our example, the user that raises a purchase order is also required to sign for receipt of the goods. 
All in all, Part (b) shows five constraints:
\begin{align*}
\text{SoD with scopes } & (s_1,s_2), (s_1,s_4), (s_3,s_5), (s_4,s_6)\\
\text{BoD with scope } & (s_1,s_3)
\end{align*}
where SoD (for \emph{separation of duty}) is a binary constraint  meaning that the steps have to be assigned to different users 
and BoD (for \emph{binding of duty}) is a binary constraint meaning that the steps have to be assigned to the same user. 
SoD and BoD constraints can be found in~\cite{ansi-rbac04}.

Let $S=\set{s_i:\ i=1,2,\dots ,6}$ be the set of steps described in Figure~\ref{fig:workflow-example-steps}.
To complete a WSP specification in this example, we introduce a set  $U=\{u_i:\ i=1,2,\dots , 8\}$ of users and 
describe authorisation lists: 
$$
\begin{array}{@{}ll@{}}
A(s_1)=\{u_1,u_2\}, &A(s_2)=\{u_2,u_3\},\\
A(s_3)=\{u_1,u_3\}, &A(s_4)=\{u_3,u_4\},\\
A(s_5)=\{u_3,u_4,u_5,u_8\}, &A(s_6)=\{u_5,u_6,u_7\}.
\end{array}
$$
The authorisation list $A(s_i)$ is the set of all users which can perform $s_i.$ 

It is not hard to verify that the following assignment $\pi:\  S\rightarrow U$ satisfies all the constraints and authorisations:
\begin{equation}
\label{eq:plan-example}
\begin{array}{lll}
\pi(s_1)=u_1,\ &\pi(s_2)=u_2,\ &\pi(s_3)=u_1,\ \\
\pi(s_4)=u_4,\, & \pi(s_5)=u_3,\, &\pi(s_6)=u_5.
\end{array}
\end{equation}

\subsection{General-case Workflow Satisfiability Problem}
Research on workflow satisfiability began with the seminal work of Bertino, Ferrari and Atluri~\cite{BeFeAt99} and Crampton~\cite{Cr05}.
In general, the Workflow Satisfiability Problem (WSP)\footnote{For a formal definition, see the beginning of Section \ref{sec:WSPwithUI}.} contains a set $S$ of steps, a set $U$ of users and some constraints and authorisations restricting performance of steps by users (the difference between an authorisation and a constraint is that while the former involves just one step, the latter involves at least two steps). The aim is to decide whether there is an assignment of users to all steps such that all authorisations and steps are satisfied. Such WSP instances are called satisfiable; the other WSP instances are unsatisfiable.

The WSP in all its generality is a problem with complicated constraints which may require a certain partial order in which the steps are performed, exclusion of certain steps from being performed \deleted{by a plan} under certain conditions, etc., see e.g.~\cite{BertolissiSR18,CrGuDM,CrGuWa16,dosSantosR17}. 
However, it is possible to reduce such a general WSP to a number of more basic WSPs in which there are no constraints imposing a partial order of the steps (i.e. the order in which the steps are performed is immaterial), no steps are excluded from being performed, etc., see e.g.~\cite{CrGuDM,CrGuWa16}. 
Because of the reductions and since new WSP notions introduced in this paper  are already of interest  for the basic WSP, we will restrict our attention to the basic WSP (an instance of such a WSP is depicted in Figure~\ref{fig:workflow-example}) which is described above and formally defined in Section~\ref{sec:WSPwithUI}.

\subsection{Special cases of WSP} 
Some authors studying and using the WSP restrict themselves to binary SoD and BoD constraints only, which are the most common WSP constraints~\cite{BertolissiSR15,CompagnaSPR16,dosSantosPR16}. 
Even restricted to SoD constraints only, WSP is NP-hard~\cite{WaLi10}.
Wang and Li~\cite{WaLi10} introduced relatively simple generalisations of binary SoD and BoD constraints and showed that WSP with such constraints admits efficient, fixed-parameter tractable (FPT) algorithms\footnote{For a brief introduction to FPT algorithms, see Section~\ref{sec:fpt}.} when parameterised by $k=|S|$.
Crampton et al.~\cite{CrGuYe13} introduced a more general family of WSP constraints, so-called regular constraints, for which the WSP still admits an FPT algorithm. 
While theoretically the algorithm of~\cite{CrGuYe13} is fast, it is unlikely to be so in practice as it uses the method of inclusion-exclusion which makes worst and best exponential running times close to each other. 
Cohen et al.~\cite{CoCrGaGuJo14} introduced the concept of the \emph{user-independent} (UI) constraints and demonstrated that WSP with UI constraints only can be solved in FPT time.
UI constraints form a useful constraint family as it includes not only all the constraints studied in~\cite{CrGuYe13} and~\cite{WaLi10}, but also all the constraints listed by the American National Standards Institute in~\cite{ansi-rbac04}. 

Cohen et al.~\cite{CohenCGGJ16} implemented their algorithm for UI constraints and showed that it is of practical interest. 
Karapetyan et al.~\cite{KarapetyanPGG19} introduced and studied a different kind of backtracking algorithm for the WSP with UI constraints, which turned out to be by many orders of magnitude faster than the algorithm of~\cite{CohenCGGJ16}.
Moreover, Karapetyan et al.~showed that similar concepts can be used in formulations for general-purpose solvers such as the \added{Boolean satisfiability} solver SAT4J~\cite{BerreP10}, and the resulting methods are sufficient for most of the WSP instances as long as all the constraints are UI\@. 
Crampton et al.~\cite{CrGaGuJoWa16} introduced a generalisation of UI constraints, the family of class-independent (CI) constraints, and proved that the WSP with CI constraints only admits an FPT algorithm. 
\added{CI constraints are defined in terms of equivalence relations over the set of users.
In particular, a CI constraint can request that two steps are required to be assigned to users from two different groups (i.e., departments).}
However, to use CI algorithms, the set $U$ of users has to have a hierarchical structure (e.g.\ if $U$ are users in some hierarchical organisation), which is rather restrictive.  
Dos Santos and Ranise~\cite{dosSantosR17} developed a software tool for the software company SAP which can be used to work, in particular, with the run-time version of the WSP which has UI and CI constraints. 
Apart from SAT and Constraint Satisfaction Problem solvers used in~\cite{CohenCGGJ16,KarapetyanPGG19}, many researchers studied WSP satisfiability using solvers based on Satisfiability Modulo Theories and Optimization Modulo Theories, see e.g.~\cite{BertolissiSR18,dosSantosR17}.
We refer the reader to~\cite{dosSantosR17survey,HoldererAM15} for surveys on workflow satisfiability approaches.

\subsection{This work}
So far, the research of the WSP has been focused on the special cases that admit efficient algorithms.
The imposed restrictions could limit the practicality of the developed algorithms.
The goal of this paper is to study solution methods for the unrestricted version of the WSP\@.
Concerning the implementation, we focus on the formulations for the general-purpose solvers.
Software solutions based on general-purpose solvers tend to be cheaper to develop and maintain than bespoke algorithms and thus strongly preferred by practitioners.

First, we give a novel approach to tackle the non-UI constraints.
We show that any constraint can be absorbed by authorisations, if we extend the definition of authorisations.
This absorption may increase the size of the problem, however in practice this increase is unlikely to be significant, and the theoretical properties of the problem are preserved.
Most importantly, this approach still handles UI constraints very efficiently, hence we expect that it will perform very well on real-world instances where the majority of constraints are likely to be UI and only a few constraints may require the new `absorption' technique.


To quantify by how much the problem size increases by the constraint absorption, we introduce the notion of \emph{branching factor} of a WSP constraint.
Branching factor is motivated by the basic algorithm of Karapetyan et al.~\cite{KarapetyanPGG19}, see Algorithm~\ref{alg:PBA} in Section~\ref{sec:fpt}. 
To the best of our knowledge, the only other complexity measure for a WSP constraint is diversity which was introduced by Cohen et al.~\cite{CoCrGaGuJo14}. 
Diversity was also motivated by an algorithm (and used to estimate its running time), but the algorithm in~\cite{CoCrGaGuJo14} is less efficient from the theoretical point of view: its running time for UI constraints is $\cO(3^kB_k N^{O(1)})$ rather than $\cO(B_kN^{O(1)})$ of Algorithm~\ref{alg:PBA}, where $k=|S|,$ $B_k$ is the $k$th Bell number, and $N$ is the size of the problem. 
More importantly, computational experiments in~\cite{KarapetyanPGG19} clearly demonstrated that the approach based on the concept of constraint diversity is many orders of magnitude slower than the approach based on the basic algorithm of~\cite{KarapetyanPGG19}.
Finally, establishing the diversity of a constraint is often a significant challenge~\cite{CoCrGaGuJo14} whereas identifying the branching factor is usually straightforward as we show in Section~\ref{sec:CDA}.

The main contributions of this paper are as follows: 
\begin{itemize}
	\item
	The concept of a constraint branching factor as a new practical measure of WSP constraints that enables FPT algorithm for the general-case WSP;
		
	\item
	A systematic study of formulations of the WSP for general-purpose solvers;
	
	\item
	A pseudo-random instance generator for the general WSP and a methodology to produce consistently-hard instances with various constraints;
	
	\item
	A rigorous computational study of the solution times for the general-case WSP;
	
	\item
	An empirical analysis of the difficulty of various constraints.
\end{itemize}

Apart from the first point, all these contributions are new to the journal version of the paper.



\section{Parametrised algorithms and complexity}
\label{sec:fpt}

The vast majority of non-trivial decision problems including WSP are intractable, i.e.\ NP-hard. 
One approach for solving an NP-hard decision problem is to use heuristics, but this means that the output is not always correct. 
Another approach is the use of parametrised algorithms. 
In this case, for the decision problem under consideration we assign a parameter (or, a collection of parameters which is usually aggregated to only one parameter, the sum of the original parameters). 
Examples of such parameters are the tree-width, tree-depth or other structural parameters of \added{the given graph in problems on graphs}. 
This way the decision problem $\Pi$ under consideration has two quantities: the size $N$ of the problem and its parameter $\kappa$. 
The usual choice of the parameter is such that $\kappa$ is much smaller than $N$ on the instances of $\Pi$ of interest. 
A problem together with its parameter is called a \emph{parameterised problem}.

The reason for introducing the parameter is to design an algorithm of running time $\cO(f(\kappa)N^c)$, where $f$ is a computable function of $\kappa$ only and $c$ is an absolute constant. 
Such an algorithm is called \emph{fixed-parameter tractable (FPT)}\@.  
FPT algorithms exist for several parameterised NP-hard problems, in which cases $f(\kappa)$ `absorbs' the classical computational complexity of the problem. 
A parameterised problem admitting an FPT algorithm is called an \emph{FPT problem} (or, the problem is in the class FPT)\@.
Unfortunately, there are many parameterised problems for which there are good reasons to believe that they are not FPT\@.  
In parameterised complexity, such problems are called W[1]-hard  and it is widely assumed that $\text{FPT} \neq \text{W[1]}$~\cite{CyganFKLMPPS15}.

To stress the fact that in an FPT algorithm of running time $\cO(f(\kappa)N^c)$ the function $f$ is usually exponential (it does not have to be if the problem is polynomial-time solvable) and thus $f(\kappa)$ is often the dominating factor in the running time, $\cO(f(\kappa)N^c)$ is often simplified to $\cO^*(f(\kappa))$.
For more information about parameterised algorithms and complexity, we refer the reader to the monograph~\cite{CyganFKLMPPS15}.

While parametrised complexity is a theoretical tool, it can also be used in practical studies.
We will say that an algorithm demonstrates \emph{FPT-like performance} if its running time scales consistently with the expectations for an FPT algorithm.
For example, this means that its empirical running time scales polynomially with the size of the problem $N$ if the value of the parameter $\kappa$ is fixed.

\section{WSP and its complexity}
\label{sec:WSPwithUI}

Let $S$ be a set of steps, $U$ a set of users, $C$ a set of non-unary constraints whose scopes are subsets of $S$ (called just \emph{constraints} in what follows\footnote{We will formally define a constraint later in this section.}) and $A$ an \emph{authorisation function} (or, just \emph{authorisation})  $A:\ S \rightarrow 2^U$. 
Thus, every WSP instance is given by a quadruple $(S,U,C,A)$. 
A \emph{plan} $\pi$ is a function $\pi: S\rightarrow U$.
The aim is to decide whether there is plan which is \emph{authorised} i.e.\ $\pi(s)\in A(s)$ for each $s\in S$, and \emph{eligible} i.e.\ satisfies all constraints. 
A plan is called \emph{valid} if it is authorised and eligible.
If a WSP instance $(S,U,A,C)$ has a valid plan, it is called \emph{satisfiable} and otherwise \emph{unsatisfiable}. 
It is easy to see that the WSP with only binary SoD constraints is NP-hard as it is equivalent to the well-known NP-hard \textsc{Graph List Colouring} problem~\cite{diestel}.  

Let us denote $|S|$ by $k$ and $|U|$ by $n$.
Wang and Li~\cite{WaLi10} observed that in real-world WSP instances $k$ is often much smaller than $n$. 
\added{In other words, while the instances of WSP can be large, practical values of $k$ are usually small.}
This led them to introduce parametrisation of the WSP by $k$. 
We will also study this parameterisations like several other papers, cf.~\cite{CoCrGaGuJo14,CohenCGGJ16,GuWa16,KarapetyanPGG19}.
Wang and Li showed that the WSP is W[1]-hard. 

\begin{remark}
\label{rem:Wproof}
In fact, Wang and Li~\cite{WaLi10} proved the following result.  
Let $\rho$ be a binary relation on $U$ and let $S=\set{s_1, \ldots, s_k}$.
Then the WSP is W[1]-hard if $A(s)=U$ for every $s\in S$ and $C$ consists of $k \choose 2$ constraints $c_{i,j}$ with $1\le j<i\le k$ such that a plan $\pi$ satisfies $c_{i,j}$ if and only if $(\pi(s_i),\pi(s_j))\in \rho$.
The W[1]-hardness follows by a simple reduction from the \textsc{(Graph) Independent Set} problem\footnote{Set $\rho$ to be the non-adjacency relation in the input graph.}, which is W[1]-hard.
\end{remark}

Let us now define a WSP constraint formally. 
A constraint $c$ with \emph{scope} $T \subseteq S$ is a pair $(T,\Theta_c),$ where  $\Theta_c$ is a set of functions $\theta:\ T\rightarrow U$ such that $c$ is {\em satisfied by a plan} $\pi$ if for some $\theta\in \Theta_c,$ $\theta(t)=\pi(t)$ for each $t\in T$.
For example, the scope of a (general) SoD constraint $c$ is a subset $T$ of $S$ and  $\theta\in \Theta_c$ if and only if $\theta$ is injective.\footnote{For an example of a non-binary WSP SoD constraint, see Constraint 1 in \cite[Example 1]{WaLi10}.}   
Satisfiability of essentially every real-world constraint $c$ by a plan $\pi$ can be decided without using the formal definition of a constraint (e.g., deciding whether an SoD constraint is satisfied by a plan $\pi$ can be done by a simple inspection of $\pi$).
In this paper, we assume that satisfiability of every constraint by a plan can be decided in polynomial time in $k$ and $n$.

Recall that a constraint is called {\em user-independent (UI)} if its satisfiability does not depend on the identities of the users. 
For example, SoD and BoD constraints are UI\@.
The following generalisations of SoD and BoD constraints are also UI\@.
The scope of an {\em at-least-$p$-out-of-$q$ constraint} is a subset $T$ of $S$ of size $q$
and at least $p$ users can be assigned to the steps of $T$. Analogously, the scope of {\em an at-most-$p$-out-of-$q$ constraint} is a subset $T$ of $S$ of size $q$
and at most $p$ users can be assigned to the steps of $T$.
Formally, a constraint $c=(T,\Theta_c)$ is UI if for every $\theta\in \Theta_c,$
and every permutation $f:\ U \rightarrow U$ of users there is a $\theta'\in \Theta_c$ such that $\theta'(t)=f(\theta(t))$ for each $t\in T.$ 

To naively decide whether a WSP instance $(S,U,C,A)$ has a valid plan, one can generate plans one by one and check whether one of them is valid, stopping when a valid one is found. 
However, such an algorithm is not efficient as the total number of plans is $n^k$. 
Since the WSP is W[1]-hard, it is highly unlikely that much more efficient, i.e., FPT algorithms exist for the most general case. 
However, there are FPT algorithms if all constraints are UI (and with no restrictions on authorisations). To describe such an algorithm, we will use the notion of a pattern introduced  by Cohen et al.  \cite{CoCrGaGuJo14} but we will follow the definition of patterns from Crampton et al. \cite{CrGuKa15b}. 

A {\em pattern} is a partition $p = \{S_1,\dots, S_p\}$ of $S$ into non-empty subsets $S_1, \dots, S_p$  (i.e.\ $S_1\cup \dots \cup S_p=S$ and $S_i\cap S_j=\emptyset$ for every $i\ne j$) 
called {\em blocks}. We will denote the set of all patterns of $S$ by $P(S)$\@. 
For example, if $S=\{s_1,s_2,s_3\}$ then $P(S)=\{p_1,p_2,p_3,p_4,p_5 \},$ where $p_1=\{\{s_1\},\{s_2\},\{s_3\}\}$ (every block of $p_1$ is a singleton), $p_2=\{\{s_1\},\{s_2,s_3\}\}$, $p_3=\{\{s_2\},\{s_1,s_3\}\}$, $p_4=\{\{s_3\},\{s_1,s_2\}\}$ and $p_5=\{\{s_1,s_2,s_3\}\}$ ($p_5$ has only one block).

The {\em pattern} $p$ {\em of a plan} $\pi$ is a partition of $S$ into blocks of steps such that all the steps in a block are assigned the same user but steps in different blocks are assigned different users.
For example, plan~(\ref{eq:plan-example}) has pattern $p = \{ \{ s_1, s_3 \}, \{ s_2 \}, \{ s_4 \}, \{ s_5 \}, \{ s_6 \} \}$.
The number of blocks in $p$ will be denoted by $|p|.$
By the definition of a UI constraint, a pair $\pi',\pi''$ of plans with the same pattern either both satisfy a UI constraint $c$ or both violate it~\cite{CoCrGaGuJo14}. 
Thus, it suffices to know the pattern of a plan $\pi$ (without knowing $\pi$ itself) to decide whether a UI constraint $c$ is satisfied by $\pi$, and we can say whether a \emph{pattern} $p$ of $S$  \emph{satisfies} a UI constraint $c$ (meaning that either all plans with pattern $p$ satisfy $c$ or none do). 
Thus, we have the following observation:
\begin{remark}
\label{rem1}
To decide whether a WSP instance $(S,U,C,A)$ with only UI constraints has an eligible plan, it suffices to go over all patterns of $S$ and check whether there is a pattern that satisfies all constraints in $C$ (we call such a pattern \emph{eligible}). 
\end{remark}

Recall that not every eligible plan is valid as authorisations $A$ have to be satisfied, too. 
Following Karapetyan et al.~\cite{KarapetyanPGG19}, to decide whether there exists a valid plan with an eligible pattern $p$, we construct a bipartite graph $G_p$ with partite sets $p$ (where every vertex is a block of $p$) and $U$ such that $bu\in E(G_p)$ ($b$ is a block in $p$ and $u \in U$) if and only if $u\in A(s)$ for every $s \in b$.
Observe that a plan with pattern $p$ is authorised if and only if $G_p$ contains a matching $M'$ saturating $p$, i.e. for every vertex (block) $b \in p$ there is an edge in $M'$ incident to $b$~\cite{KarapetyanPGG19}. 
Since $k \le n$, this means that a plan with pattern $p$ is authorised if and only if the size of a maximum matching $M$ of $G_p$ equals $|p|$, the number of blocks in $p$.
If $|M| = |p|$, an authorised plan $\pi$ \emph{corresponding} to $M$ can be constructed as follows: $\pi(s) = u$ if and only if $s\in b$ and $bu\in M$~\cite{KarapetyanPGG19}. 
For an example of graph $G_p$, see Figure~\ref{fig:g-p}.

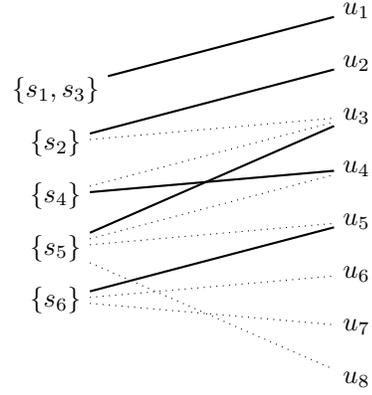
\begin{figure}[tb]
	\centering
	\begin{tikzpicture}[
    	xscale=1,
        yscale=0.7,
		main node/.style={},
		user node/.style={},
        secondary/.style={thick, gray}]

		\node[main node] (13)  at (0, 4.5) {$\set{s_1, s_3}$};
		\node[main node] (2) at (0, 3.5) {$\set{s_2}$};
		\node[main node] (4) at (0, 2.5) {$\set{s_4}$};
		\node[main node] (5) at (0, 1.5) {$\set{s_5}$};
		\node[main node] (6) at (0, 0.5) {$\set{s_6}$};

		\node[user node] (u1) at (4, 6) {$u_1$};
		\node[user node] (u2) at (4, 5) {$u_2$};
		\node[user node] (u3) at (4, 4) {$u_3$};
		\node[user node] (u4) at (4, 3) {$u_4$};
		\node[user node] (u5) at (4, 2) {$u_5$};
		\node[user node] (u6) at (4, 1) {$u_6$};
		\node[user node] (u7) at (4, 0) {$u_7$};
		\node[user node] (u8) at (4, -1) {$u_8$};

		\path[dotted]
      	(2) edge (u3)
    		(4) edge (u3)
    		(5) edge (u4)
    		(5) edge (u5)
    		(5) edge (u8)
    		(6) edge (u6)
    		(6) edge (u7);
    		
		\path[thick]
    		(13) edge (u1)
      	(2) edge (u2)
    		(4) edge (u4)
    		(5) edge (u3)
    		(6) edge (u5);

	\end{tikzpicture}
\caption{
	Graph $G_p$ constructed for the example given in Section~\ref{sec:intro}.
	The authorisations in this example are as follows:
	$A(s_1) = \set{u_1, u_2}$, $A(s_2) = \set{u_2, u_3}$, $A(s_3) = \set{u_1,u_3}$, $A(s_4) = \set{u_3,u_4}$, $A(s_5) = \set{u_3,u_4,u_5,u_8}$ and $A(s_6) = \set{u_5,u_6,u_7}$.  
	The pattern is $\set{\set{s_1, s_3}, \set{s_2}, \set{s_4}, \set{s_5}, \set{s_6}}$.
	The matching shown in bold lines corresponds to the {valid} plan {(\ref{eq:plan-example})}: $\pi(s_1) = \pi(s_3) = u_1$, $\pi(s_2) = u_2$, $\pi(s_4) = u_4$, $\pi(s_5) = u_3$, $\pi(s_6) = u_5$.
    }
\label{fig:g-p}
\end{figure}

This leads to a simple WSP algorithm using patterns whose pseudo-code is given in Algorithm~\ref{alg:PBA}, where all constraints in C are UI and UNSAT indicates that the input instance is unsatisfiable. 
A more sophisticated and much more efficient in computational experiments pattern-based backtracking algorithm was studied by Karapetyan et al.~\cite{KarapetyanPGG19}.

\SetKwInOut{Input}{input}
\SetKwInOut{Output}{output}

\begin{algorithm}[htb]
\caption{Pattern Basic Algorithm (WSP)}
\label{alg:PBA}

\Input {WSP instance $W = (S, U, A, C)$}
\Output {Valid plan $\pi$ or UNSAT}
\For {$p\in P(S)$}
{
	Given $p$, $U$ and $A$, compute $G_p$\; 
	Compute a maximum matching $M$ in $G_p$\;
	\If {$p$ is eligible and $|M|=|p|$}
	{
		\Return {plan $\pi$ corresponding to $M$}\;
	}
}
\Return {UNSAT}
\end{algorithm}

Note that $|P(S)|$, i.e., the number of partitions of a set of size $k$, equals the $k$th Bell number\footnote{All logarithms in this paper are of base 2.} $B_k\le k!=\cO(2^{k\log k})$. 
Now it is not hard to see that the running time of Algorithm~\ref{alg:PBA} is $\cO^*(2^{k\log k})$. 
More advanced algorithms in~\cite{CoCrGaGuJo14,KarapetyanPGG19} for the WSP with only UI constraints are also of running time $\cO^*(2^{k\log k})$.
In fact, the time $\cO^*(2^{k\log k})$ is highly likely to be optimal: Cohen et al.~\cite{CoCrGaGuJo14} proved that unless the Exponential Time Hypothesis (ETH) fails, there is no algorithm of running time $\cO^*(2^{o(k\log k)})$ for the WSP with UI constraints; Gutin and Wahlstr{\"{o}}m~\cite{GuWa16} showed that unless the Strong ETH fails, there is no algorithm of running time $\cO^*(c^{k\log k})$ for any constant $c<2$ for the WSP with UI constraints only.\footnote{
The ETH claims that 3-SAT with $n$ variables cannot be solved in time $O(2^{o(n)})$~\cite{ImPa01}. 
The Strong ETH claims that SAT with $n$ variables cannot be solved in time $O(c^{n})$ for any $c<2$ \cite{ImPaZa01}.
Note that both ETH and Strong ETH are conjectures, but they are often used to show (relative) optimality of algorithm's complexity.}

\section{Handling non-UI constraints}
\label{sec:CDA} 

As shown in Section~\ref{sec:WSPwithUI}, WSP with UI constraints can be solved in FPT time.
However, some practical constraints may not be UI\@.
In Section~\ref{sec:examples-of-non-UI-constraints}, we give a few examples of non-UI constraints.
To handle non-UI constraints efficiently, we introduce a generalisation of WSP in Section~\ref{sec:wsp-cda}; this generalisation lets us absorb non-UI constraints into authorisations.
There is a cost associated with this absorption as it may increase the problem size, and we introduce an appropriate constraint complexity measure in Section~\ref{sec:absorption}.
As long as the measure is relatively small, the problem can be solved in FPT time.
Finally, in Section~\ref{sec:complexity-of-practical-non-UI-constraints}, we show that the new measure is indeed relatively small for most of the practical constraints, and hence the results are of real importance.

\subsection{Examples of non-UI constraints}
\label{sec:examples-of-non-UI-constraints}

Below we define three non-UI constraint types, none of which can be handled by the FPT algorithms known from the literature.
\added{These three types are given to illustrate and evaluate our approach.  
Further research is needed to study applications of these and other types of non-UI constraints.}

\textbf{Super-User At-Least constraint (SUAL).}
Let $T \subseteq S$ be a set of steps, $h$ a constant and $X \subset U$ be a set of users called `super users'.
The SUAL constraint requires that the steps $T$ have to be assigned only to super users if the number of users assigned steps $T$ is less than or equal to $h$.
In other words, at least $h + 1$ users have to be assigned to $T$, or the users have to be super users.

For example, $X$ is the set of senior employees, or employees with a certain training.
If the number of employees assigned steps $T$ is small, then they need to be chosen from $X$.

\textbf{Wang-Li (WL).}
This is a generalisation of Constraint~2 in \cite[Example~1]{WaLi10}, where $|T| = 2$.
Let $T \subseteq S$ be a set of steps and $U_1$, $U_2$, \ldots, $U_d$ non-intersecting sets of users.
The Wang-Li constraint requires that all the steps $T$ are assigned users from $U_i$ for some $i$.

For example, $U_1$, $U_2$, \ldots, $U_d$ are departments in an organisation, and we need to have the entire subflow $T$ to be performed within a single department.

\textbf{Assignment-Dependent Authorisation (ADA).}
Let $U_1$ and $U_2$ be two sets of users, and $s_1$ and $s_2$ two steps.
Then, if $s_1$ is assigned to a user from $U_1$ then $s_2$ has to be assigned to a user from $U_2$.

For example, $U_1$ is the set of users in the defence department.
Then, if anyone from $U_1$ is assigned to $s_1$ then someone from the security department needs to perform $s_2$.

\bigskip

\added{Note that the above constraints are all non-UI.
Indeed, it is easy to find examples of each of the above three constraint types, such that some plan $\pi$ satisfies the constraint but a certain permutation of users in $\pi$ will not satisfy the constraint.
For example, consider a WSP instance with a SUAL constraint and a valid plan $\pi$ that assigns at most $h$ users to $T$ and all those users are from $X$.  If we replace one of those users with a user not in $X$ then the plan will not satisfy the constraint.}

\subsection{WSP with context-dependent authorisations}
\label{sec:wsp-cda}

A careful consideration of Algorithm~\ref{alg:PBA} shows that the algorithm may remain FPT even if we extend WSP allowing certain flexibility of the authorisations.
We will use the term \emph{context-dependent authorisations} (CDAs) for such authorisations, and WSP-CDA for the WSP extended with CDAs.

Formally, CDAs are defined as $\cA = \{ \cA_p :\ p \in P(S) \}$.
Each $\cA_p$ is a {``disjunctive''} set of authorisation functions $\{A_{p,1},\dots ,A_{p,d_p}\}$ ($d_p = |\cA_p|$ and each $A_{p,i}$ maps $S$ to $2^U$). 
We will call $\cA_p$ a \emph{$p$-authorisation family} and $\cA$ an \emph{authorisation family}.
In cases where $\cA_p$ is the same for all $p \in P(S)$, we will omit subscript $p$: $\cA = \{ A_1, A_2, \ldots, A_d \}$.
For a list of all the authorisation-related notations, see Table~\ref{tab:authorisation-notations}.

\begin{table}[htb]
\renewcommand{\arraystretch}{1.3}
\begin{center}
\begin{tabular}{@{}p{7em}lp{10em}@{}}
\toprule
Authorisation list & $A(s) \subseteq U$ & Set of users allowed to perform step $s$ \\
Authorisation function & $A : S \rightarrow 2^U$ & Authorisation lists for every step \\
$p$-authorisation family & $\cA_p = \{ A_{p,1}, \ldots, A_{p,d_p} \}$ & A set of permitted authorisation functions given pattern $p$ \\
Authorisation family & $\cA = \{ \cA_p : p \in P \}$ & A set of $p$-authori\-sa\-tion families for every pattern\\
\bottomrule
\end{tabular}
\end{center}

\caption{Authorisation-related notations.}
\label{tab:authorisation-notations}
\end{table}

A plan $\pi$ is \emph{authorised} for a WSP-CDA instance $(S, U, \cA, C)$ if and only if it is authorised for at least one of the authorisation functions in $\cA_p$, where $p$ is the pattern of plan $\pi$.
Similarly, a pattern $p$ is \emph{authorised} for a WSP-CDA instance $(S,U,C,\cA)$ if it is authorised for at least one of the authorisation functions in $\cA_p$. 

As an example, consider the instance of WSP given in Figure~\ref{fig:workflow-example}. 
Recall that it has $S = \{ s_1, \dots, s_6 \}$, $U = \{u_1, \dots, u_8\}$, and five constraints 
$$
(s_1,s_2,\ne), (s_1,s_4,\ne), (s_3,s_5,\ne), (s_4,s_6,\ne), (s_1,s_3,=).
$$ 
Let us modify the WSP instance by letting $\cA=\{ A_1, A_2 \}$ as follows:
\begin{align*}
& A_1(s_2) = \{u_2\} && A_2(s_2) = \{u_3\} \\
& A_1(s_6) = \{u_5, u_6\} && A_2(s_6) = \{u_5, u_6, u_7\} \\
& A_1(s_i) = U \text{ for } i = 1, 3, 4, 5 && A_2(s_i) = U \text{ for } i = 1, 3, 4, 5
\end{align*}
These CDAs describe the following logic.
If user $u_2$ is assigned to approve the purchase order (step $s_2$) then only users $u_5$ and $u_6$ are allowed to approve payment (step $s_6$); however if user $u_3$ is assigned to approve the purchase order, user $u_7$ is also allowed to approve the payment.
It is not hard to see that plan~(\ref{eq:plan-example}) satisfies $A_1$ but does not satisfy $A_2$.
Since~(\ref{eq:plan-example}) satisfies at least one authorisation family, it is an authorised plan for our WSP-CDA instance.

Consider Algorithm~\ref{alg:PDABA}, which is 
a straightforward modification of Algorithm~\ref{alg:PBA}, which is also designed for UI constraints but can handle CDAs\@.
Here $G_{p,i}$ is the bipartite graph corresponding to the authorisation function $A_{p,i}$.
It is not hard to see that Algorithm~\ref{alg:PDABA} 
is still efficient as long as the $p$-authorisation families are relatively small.
We will discuss the complexity of WSP-CDA in more detail in Section~\ref{sec:absorption}.

\begin{algorithm}[tb]
\caption{Pattern Basic Algorithm (WSP-CDA)}
\label{alg:PDABA}

\Input {WSP-CDA instance $W = (S, U, \cA, C)$, where $\cA = \{ \cA_p :\ p \in P(S) \}$ and $\cA_p = \set{A_{p,1}, \ldots, A_{p,d_p}}$ for each $p \in P(S).$}
\Output {Valid plan $\pi$ or UNSAT}
\For {$p\in P(S)$ and $i \in [d_p]$}
{
Given $p$, $U$ and $A_{p,i}$, compute $G_{p,i}$\; 
Compute a maximum matching $M$ in $G_{p,i}$\;
\If {$p$ is eligible and $|M|=|p|$}
  {
    \Return {plan $\pi$ corresponding to $M$}\;
  }
}
\Return {UNSAT}  
\end{algorithm}

\subsection{Absorption of non-UI constraints}
\label{sec:absorption}

While the WSP-CDA seems to be of interest in its own right, the main reason for introducing CDAs is to efficiently handle non-UI constraints.
We will show below that we can absorb \emph{all} non-UI constraints into CDAs\@. 
The absorption may increase the number of authorisation functions in the $p$-authorisation families.
We introduce the notion of a \emph{branching factor} of a constraint as a measure of how much does its absorption increase the size of the $p$-authorisation families.
Our absorption is efficient for constraints of small branching factor, especially those of branching factor~1. 
In fact, constraints of branching factor~1 are direct generalisations of UI constraints, see Proposition~\ref{th:UIbranch} below.

To link CDAs with practical constraints, we need to introduce the concepts of authorisation families intersection and plan-equivalent instances.
Let $$\cA'=\set{\cA'_p:\ p\in P(S)} \mbox{ and } \cA''=\set{\cA''_p:\ p\in P(S)}$$
be a pair of authorisation families, where each $$\cA'_p=\set{A'_{p,1},\dots ,A'_{p,d'_p}} \mbox{ and } \cA''_p=\set{A''_{p,1},\dots ,A''_{p,d''_p}}$$ is a $p$-authorisation family. 
Then we say that the \emph{intersection} $\cA:=\cA'\cap \cA''$ is an authorisation family such that  $$\cA_p=\set{A'_{p,i}\cap A''_{p,j}:\ i\in [d'_p], j\in [d''_p]},$$ where each
$A'_{p,i}\cap A''_{p,j}$ is an authorisation function such that $$A'_{p,i}\cap A''_{p,j}(s):=A'_{p,i}(s)\cap A''_{p,j}(s).$$ 

A pair $(S,U,C',{\cA}')$ and $(S,U,C'',{\cA}'')$ of WSP-CDA instances are called
\emph{plan-equivalent} if for every plan $\pi: S\rightarrow U,$  $\pi$ is valid for one of them if and only if it is valid for the other. 

A constraint $c$ is called {\em $m$-branching} if there is an authorisation family ${\cA}^{(c)}=\set{{\cA}^{(c)}_p:\ p\in P(S)},$ where each $${\cA}^{(c)}_p= \set{A^{(c)}_{p,1},\dots ,A^{(c)}_{p,m_p}}$$
is a $p$-authorisation family with $m_p\le m$, such that every WSP-CDA instance $(S,U,C,\cA)$ with $c\in C$ is plan-equivalent to the WSP-CDA instance $(S,U,C\setminus\{c\},\cA\cap {\cA}^{(c)}).$
We call ${\cA}^{(c)}$ an {\em authorisation family} of $c.$ Let $m(c,{\cA}^{(c)}):=\max_{p\in P(S)}m_p.$ 
The {\em branching factor} $m(c)$ of $c$ is the minimum of $m(c,{\cA}^{(c)})$ over all authorisation families ${\cA}^{(c)}$ of $c$.
We have the following:

\begin{proposition}
\label{th:mbranch}
Every WSP constraint $c$ is $m$-branching for some $m$.
\end{proposition} 
\begin{proof}
Consider  ${\cA}^{(c)}_p=\set{A^{(c)}_{p,\pi}:\ \pi \in \Pi_p},$ where $\Pi_p$ 
is the set of all plans with pattern $p$ and $A^{(c)}_{p,\pi}(s)=\pi(s)$ for every $s\in S.$ We can set $m=\max_{p\in P(S)} |{\cA}^{(c)}_p|.$
\end{proof}

The above proposition shows that every WSP constraint, even if it is a non-UI constraint, can be absorbed into CDAs.
For some WSP constraints, though, the branching factor will be impractically large.
Remarks~\ref{rem:SUAL-branching-factor},~\ref{rem:WL-branching-factor} and~\ref{rem:ADA-branchng-factor} below demonstrate that the non-UI constraints introduced in Section~\ref{sec:examples-of-non-UI-constraints} have relatively small branching factors.


\begin{proposition}
\label{rem:SUAL-branching-factor}
Every SUAL constraint $c$ has branching factor 1.
\end{proposition}
\begin{proof}
We can define a $p$-authorisation family for $c$ as follows: $\cA^{(c)}_p=\set{A^{(c)}_p}$ such that for every $s \in S$,
$A^{(c)}_p(s) = X$ if $p$ has at most $h$ blocks containing all steps of $T$ (and possibly other steps)
and $A^{(c)}_p(s) = U$, otherwise.  
Note that $m(c) = 1$, but $c$ is not UI (see Proposition~\ref{th:UIbranch}). 
Indeed, already for a set $X$ of size 1, $c$ has dependence on users. 
\end{proof}

\begin{proposition}
\label{rem:WL-branching-factor}
Every WL constraint $c$ has branching factor at most $d$, and this is a sharp bound.
\end{proposition}
\begin{proof}
Define an authorisation family for $c$ as follows:
$$
\cA^{(c)} = \set{A^{(c)}_{1}, \dots, A^{(c)}_{d}},
$$ 
where for $j \in [d]$, $A^{(c)}_{j}(t) = U_j$ for all $t \in T$ and $A^{(c)}_{j}(s) = U$ for all $s \in S\setminus T$\@. 
Thus, $m(c) \le d$.

To prove that this is a sharp bound, consider a WSP-CDA instance $(S,U,\{A\},\{c\}),$ where 
$S = \{ s_1, s_2 \}$, $U = \{ u_i : i = 1, 2, \ldots, 2d\}$, $c$ is a WL constraint with $T=S$ and 
$U_i = \{ u_{2i-1}, u_{2i} \}$ for $i = 1, 2, \dots, d$, and $A(s_1)=A(s_2)=U.$

Let us assume that $m(c) < d$ i.e.\ that there exists an authorisation family $\cA^{(c)} = \{ \cA^{(c)}_p : p \in P(S) \}$ such that $|\cA^{(c)}_p| < d$ for every $p \in P(S)$
and  $(S,U,\{A\},\{c\})$ is plan-equivalent to $(S,U,\cA,\emptyset)$, where $\cA = \cA^{(c)} \cap \{ A \}$.
Observe that, according to our assumption, $|\cA_p| < d$ for every $p \in P(S)$.

Consider the pattern $p = \{ \{s_1\}, \{s_2\} \}$, which means that the two steps have to be assigned to distinct users.
For every $i \in [d]$, there are exactly two such plans, $\pi'_i$ and $\pi''_i$, where $\pi'_i(s_1) = u_{2i-1}$, $\pi'_i(s_2) = u_{2i}$, $\pi''_i(s_1) = u_{2i}$ and $\pi''_i(s_2) = u_{2i-1}$.
Then the $p$-authorisation family $\cA_p$ has to allow exactly these $2d$ plans.
Let $\cA_p = \{ A_1, A_2, \ldots, A_r \}$.
Since $r < d$, there exists an authorisation function $A_t$, $t \in [r]$, that allows more than two plans.
Since there can be only two plans that involve the users from a particular user group $U_i$, we conclude that, for some $s \in T$, there exist $u_i, u_j \in U$ such that they belong to different user groups and $\{ u_i, u_j \} \in A_t(s)$.
Let $s'$ be the other step, i.e. $\{ s, s'\} = T$.
Note that $A_t(s')$ has to be non-empty and let $u \in A_t(s')$.
Then the following two plans $\pi$ and $\pi^*$ are authorised by $A_t$: $\pi(s') = u$ and $\pi(s) = u_i$, and $\pi^*(s') = u$ and $\pi^*(s) = u_j$.
Since $u_i$ and $u_j$ belong to different user groups, at least one of the plans $\pi$ and $\pi^*$ assigns users from different user groups to the steps of $T$.
This is a contradiction, hence our assumption is wrong.
\end{proof}

\begin{proposition}
\label{rem:ADA-branchng-factor}
Every ADA constraint $c$ has branching factor at most 2, and this is a sharp bound.
\end{proposition}
\begin{proof}
We can define $\cA^{(c)} = \{ A_1, A_2 \}$ as follows:
$$
A_1(s_1) = U_1,\ A_1(s_2) = U_2,\ A_1(s) = U \text{ for } s \notin \{ s_1, s_2 \}
$$
$$
A_2(s_1) = U \setminus U_1,\ A_2(s) = U \text{ for } s \neq s_1
$$

To see that it is a sharp bound, consider a WSP-CDA instance $(S,U,\{A\},\{c\}),$ where $S=\{s_1,s_2\},$ $U = \{ u_1, u_2, u_3 \}$, $c$ is an ADA constraint with $T=S,$ $U_1 = \{ u_1 \}$ and $U_2 = \{ u_2 \}$,
and $A(s_i)=U$ for $i=1,2.$ 
Assume that $m(c) = 1$ i.e. there exists $\cA^{(c)} = \{ \cA^{(c)}_p : p \in P(S) \}$ such that $\cA^{(c)}_p = \{ A_p \}$ for each $p \in P(S)$ and $(S,U,\{A\},\{c\})$ is plan-equivalent to $(S,U,\cA^{(c)},\emptyset).$ 

Consider a pattern $p = \{ \{s_1\}, \{s_2\} \}$ which assigns different users to the steps.
Then a plan $\pi(s_1) = u_1$, $\pi(s_2) = u_2$ is authorised by $A^{(c)}_p$.
Also, a plan $\pi'(s_1) = u_2$, $\pi'(s_2) = u_3$ is authorised by $A^{(c)}_p$.
Hence $\{ u_1, u_2 \} \subseteq A^{(c)}_p(s_1)$ and $\{ u_2, u_3 \} \subseteq A^{(c)}_p(s_2)$.
Then a plan $\pi''(s_1) = u_1$, $\pi''(s_2) = u_3$ is also authorised by $A^{(c)}_p$, however $\pi''$ does not satisfy $c$.
This is a contradiction, hence our assumption is wrong, and $m(c) = 2$.
\end{proof}

The next result shows that constraints of branching factor 1 (including SUAL constraints) generalise UI constraints. 

\begin{proposition}
\label{th:UIbranch}
Every UI constraint is of branching factor 1. 
\end{proposition} 
\begin{proof}
Let $c$ be a UI constraint and let $P^{(c)}$ be the set of patterns in $P(S)$ which satisfy $c$. 
By Remark \ref{rem1}, we can define a $p$-authorisation family for $c$ as follows: $\cA^{(c)}_p=\set{A^{(c)}_p}$, where for every $s \in S$, we have $A^{(c)}_p(s)=U$ if $p \in P^{(c)}$ and $A^{(c)}_p(s)=\emptyset$, otherwise.
\end{proof} 
 
For the purpose of building an algorithm for WSP-CDA, we can interpret $m$-branching constraints as follows. 
We will say that a constraint $c$ is {\em absorbed} into $\cA,$ i.e., $c$ is removed from $C$ at the cost of replacing $\cA$ by $\cA\cap {\cA}^{(c)}.$ 
The size of $\cA\cap {\cA}^{(c)}$ will depend on the branching factor of $c,$ the best case being $m(c)=1.$ 
Since $(S,U,\cA\cap {\cA}^{(c)},C\setminus \{c\})$ is a WSP-CDA instance, if $(S,U,\cA,C)$ has more than one constraint,  we can similarly absorb all of them into $\cA$ one by one. 
Algorithm~\ref{alg:PDABA} shows that there is no need to absorb UI constraints and so our absorption approach may be efficient if both the number of non-UI constraints and their branching factors are relatively small.
We will formalise this remark as follows.



\begin{theorem}
\label{th:abs1}
Let $W = (S, U, \cA, C)$ be a WSP-CDA instance, where $\ell$ constraints have branching factors between 2 and $m$, and the rest of the constraints have branching factor 1.
Then $W$ can be solved in time $\cO^*(m^{\ell} 2^{k\log k})$. 
\end{theorem} 
\begin{proof} 
Let each $p$-authorisation family $\cA_p$ of $W$ contain at most $d$ authorisations functions for every $p\in P(S).$ 
Let us first absorb the $\ell$ non-UI constraints one by one. 
The resulting WSP-CDA instance $W'$ will have each $p$-authorisation family with at most $dm^{\ell}$  authorisation functions.
Thus, the running time to absorb all non-UI constraints will be $\cO^*({\ell}dm^{\ell} 2^{k \log k})=\cO^*(m^{\ell} 2^{k \log k})$.
Now we can apply Algorithm \ref{alg:PDABA} to $W'$. 
The running time of Algorithm \ref{alg:PDABA} will be $\cO^*(m^{\ell}2^{k\log k})$.
Hence, the total running time
will be $\cO^*(m^{\ell} 2^{k\log k})$.
\end{proof}

Recall that it is widely assumed that $\text{FPT} \ne \text{W[1]}$.

\begin{corollary}
\label{cor:wsp-cda-properties}
Let $W=(S,U,\cA,C)$ be a WSP-CDA instance in which all but $\ell$ constraints are of branching factor 1 and the $\ell$ constraints are of branching factor at least 2 at most $m$.
We have the following:
\begin{enumerate}[(a)]
	\item 
	If $m^\ell = \cO^*(f(k))$ for some computable function $f$ of $k$ only, then $W$ can be solved in FPT time. 
	
	\item
	If either $m=1$ or $\ell$ and $m$ are constants, then $W$ can be solved in time $\cO^*(2^{k\log k})$.

	\item 
	Consider constraints $c_{i,j}$, $1 \le j < i \le k$ defined in Remark~\ref{rem:Wproof}. 
	Unless $\text{FPT} = \text{W[1]}$, there is no computable function $h$ of $k$ only such that $m(c_{i,j}) \le h(k)$ for any $1 \le j < i \le k$.
\end{enumerate}
\end{corollary}
\begin{proof}
Claims (a) and (b) immediately follow from Theorem~\ref{th:abs1}. To prove (c), suppose that there is computable function $h$ of $k$ only such that $m(c_{i,j})\le h(k)$ for some $1\le j<i\le k.$
By symmetry of constraints $c_{i,j},$ we have $m(c_{i,j})\le h(k)$ for all $1\le j<i\le k.$ Consider the WSP with $S=\set{s_1,\dots ,s_k}$, $C=\set{c_{i,j}:\ 1\le j<i\le k}$ and $A(s_i)=U$ for every $i=1,\dots ,k$. 
Note that $\ell=|C|={k \choose 2}$.
By (a), the WSP is FPT. However, by Remark \ref{rem:Wproof} the WSP is W[1]-hard. If there is no contradiction, FPT must be equal to W[1]. 
\end{proof}




\subsection{Branching factor of practical non-UI constraints}
\label{sec:complexity-of-practical-non-UI-constraints}

Below we show that the branching factor of any constraint limited to either a fixed number of steps or a fixed number of users is at most polynomial in $k$ and hence can be handled efficiently.

\begin{lemma}
\label{th:t-steps}
A constraint with a scope of size $t$ has a branching factor $\cO(n^t)$.
\end{lemma}
\begin{proof}
Consider a constraint $c$ with scope $T \subseteq S$, where $|T| = t$.
Note that if $\pi$ and $\pi'$ are two plans, $\pi$ satisfies $c$ and $\pi'(s) = \pi(s)$ for every $s \in T$ then $\pi'$ also satisfies $c$.

Let $\sigma_i : T \rightarrow U$ for $i = 1, 2, \ldots, r$ be all the assignments of steps $T$ to $U$ that satisfy constraint $c$; note that $c$ is not concerned with the assignment of the other steps.
Then the authorisation family of $c$ can be defined as $\cA^{(c)} = \{ A_1, A_2, \ldots, A_r \}$, where $A_i(s) = \{ \sigma_i(s) \}$ for every $s \in T$ and $i\in  [r]$, for every pattern $p \in P$.
Observe that there are $n^t$ assignments of steps of $T$ to the users $U$, hence $r = \cO(n^t)$.
\end{proof}

By following a similar logic, we can arrive to the conclusion that a constraint that involves only $t$ users has a branching factor $O(t^k)$ i.e.\ it is exponential in $k$.
However, the next lemma shows that we can achieve a better lower bound on the branching factor of such a constraint.

\begin{lemma}
\label{th:t-users}
Any constraint that involves only $t$ users in its definition has a branching factor $\cO((k+1)^t)$.
\end{lemma}
\begin{proof}
Consider a constraint $c$ which only involves users $U' \subseteq U$, where $|U'| = t$.
We will say that plans $\pi$ and $\pi'$ are \emph{equivalent} iff $\pi'(s) = \pi(s)$ for every step $s$ such that $\pi(s) \in U'$ or $\pi'(s) \in U'$.
Observe that, if $\pi$ and $\pi'$ are equivalent then $\pi$ satisfies $c$ iff $\pi'$ satisfies $c$.


Consider an arbitrary pattern $p \in P(S)$ and the set of plans $\Pi$ whose pattern is $p$.
Observe that $\Pi$ can be partitioned into equivalence classes by our equivalence relation.
We will say that an equivalence class satisfies $c$ if it includes at least one plan that satisfies $c$.
(Note that, either, all the plans in an equivalence class satisfy $c$ or all of them do not satisfying $c$.)

According to the definition of pattern, a user can only be assigned to at most one block in $p$.
Thus, there exist $\cO((k+1)^t)$ assignments of users $U'$ to blocks in $p$.
Each such assignment defines exactly one equivalence class, hence there are $\cO((k+1)^t)$ equivalence classes within $\Pi$.

Let $\Pi_1, \Pi_2, \ldots, \Pi_r$ be all the equivalence classes within $\Pi$ that satisfy constraint $c$.
For each equivalence class $\Pi_i$, we will construct an authorisation function $A_i$.
Specifically, for each step $s$ such that $\pi(s) \in U'$, for some $\pi \in \Pi_i$, let $A_i(s) = \{ \pi(s) \}$.
(Recall that if $\pi(s) \in U'$ for some $\pi \in \Pi_i$ then $\pi'(s) = \pi(s)$ for every $\pi' \in \Pi_i$ by the definition of equivalence class.)
Also, for each step $s$ such that $\pi(s) \notin U'$ for some $\pi \in \Pi_i$, let $A_i(s) = U \setminus U'$.

It is easy to see that exactly the plans in $\Pi_i$ are allowed by the authorisation function $A_i$.
Then the $p$-authorisation family $\cA^{(c)}_p = \{ A_1, A_2, \ldots, A_r \}$ allows exactly the plans that satisfy $c$ and have pattern $p$.
Finally, an authorisation family $\cA^{(c)} = \{ \cA_p^{(c)} : p \in P(S) \}$ absorbs constraint $c$.
Since the number of authorisation functions in each $p$-authorisation family is bounded by $\cO((k+1)^t)$, the branching factor of $c$ is $\cO((k+1)^t)$.
%
\end{proof}

Lemma~\ref{th:t-users} demonstrates an important tendency.
Even if a constraint is user-dependent, the patterns may still play an important role.
In other words, the patterns on their own are insufficient to describe a non-UI constraint but they may help in describing it.
If we ignore the patters, the number of authorisation functions needed to describe a constraint that involves $t$ users, is $O(t^k)$, 
however these authorisation functions are distributed between several $p$-authorisation families, and each of them gets at most $O((k+1)^t)$ authorisation functions.

From Lemmas~\ref{th:t-steps}, \ref{th:t-users} and Corollary~\ref{th:mbranch}, we immediately get the following result.
\begin{theorem}
\label{th:fixed-size-constraints-fpt}
Let $t$ be a constant. 
Let $W = (S, U, A, C)$ be a WSP instance such that $C = C_1 \cup C_2 \cup C_3$, where $C_1$ contains only constraints of branching factor 1, $C_2$ contains only constraints each involving at most $t$ users and $C_3$ contains only constraints each involving $t$ steps.
Then $W$ can be solved in FPT time if $|C_2|$ is a computable function of $k$ only and $|C_3|$ is a constant.
\end{theorem}

Theorem~\ref{th:fixed-size-constraints-fpt} is a powerful result; it shows that the WSP is FPT for any `small' constraints as long as the number of the constraints is reasonable.

While Theorem~\ref{th:fixed-size-constraints-fpt} may create the impression that any compact WSP instances are FPT, this is not the case, see the example in Corollary~\ref{cor:wsp-cda-properties}(c).

\section{Formulations of WSP for general-purpose solvers}

In this section we discuss the WSP solution representations that can be used in various off-the-shelf solvers, and then talk about encodings of common constraints.

\subsection{Solution representations}

There are three solution representations known from the WSP literature.

\emph{User-Dependent Pseudo-Boolean} (UDPB) representation consists of Boolean variables $x_{s,u}$ for $s \in S$ and $u \in U$, where $x_{s,u} = 1$ iff step $s$ is assigned to user $u$.\footnote{We will use constants True and 1, and False and 0, interchangeably.}.
	This formulation was first used in \cite{CohenCGGJ16} and it enables convenient formulation of many constraints.
	It involves only Boolean variables and, thus, can be used with practically any solver.
	So far, none of the solvers we tested demonstrated FPT-like running times with the UDPB solution representation even for WSP with UI constraints; this solution representation seems to hide the inherent structure of such instances.
	
\emph{Pattern-Based Pseudo-Boolean} (PBPB) representation was introduced in \cite{KarapetyanPGG19} and it extends the UDPB formulation with auxiliary Boolean variables $M_{s', s''}$ for $s', s'' \in S$ such that $M_{s',s''} = 1$ iff $s'$ and $s''$ are assigned the same user.
	The $M$-variables alone define a pattern, see Section~\ref{sec:WSPwithUI}, hence the solver does not need to assign values to the $x$-variables to search for an eligible pattern.
	Thus, with the PBPB solution representation and the right branching, a general-purpose solver can solve an instance of WSP with UI constraints in FPT time, which is consistent with the experimental results.
	
	The PBPB representation still uses only Boolean variables and can be used with practically any solver.
	
	The PBPB formulation requires the following constraints:
\begin{equation}
	\label{eq:mx-symmetric}
	M_{s_1, s_2} = M_{s_2, s_1} 
		\quad \forall s_1, s_2 \in S, \\
	\end{equation}
	\begin{equation}
	\label{eq:mx-diagonal}
	M_{s, s} = 1
		\quad \forall s \in S, \\
	\end{equation}
	\begin{equation}
	\label{eq:mx-link-m1}
	M_{s_1, s_2} \land M_{s_2, s_3} \implies M_{s_1, s_3}
		\quad \forall s_1, s_2, s_3 \in S, \\
	\end{equation}
	\begin{equation}
	\label{eq:mx-link-m2}
	\lnot M_{s_1, s_2} \land M_{s_2, s_3} \implies \lnot M_{s_1, s_3}
		\quad \forall s_1, s_2, s_3 \in S, \\
	\end{equation}
	\begin{equation}
	\label{eq:mx-link1}
	M_{s_1, s_2} \implies x_{s_1, u} = x_{s_2, u}
		\quad \forall s_1, s_2 \in S,\ \forall u \in U, \\
	\end{equation}
	\begin{equation}
	\label{eq:mx-link2}
	\lnot M_{s_1, s_2} \implies \lnot x_{s_1, u} \lor \lnot x_{s_2, u}
		\quad \forall s_1 \neq s_2 \in S,\ \forall u \in U.
	\end{equation}
	Constraints~(\ref{eq:mx-symmetric}) and~(\ref{eq:mx-diagonal}) ensure symmetry of the $M$-variables and correct assignment of dummy variables $M_{s,s}$.
	Constraints~(\ref{eq:mx-link-m1}) and~(\ref{eq:mx-link-m2}) implement the transitive properties: if $s_1$ and $s_2$ are assigned the same user and $s_2$ and $s_3$ are assigned the same user then $s_1$ and $s_3$ are assigned the same user; also, if $s_1$ and $s_2$ are assigned different users but $s_2$ and $s_3$ are assigned the same user then $s_1$ and $s_3$ are assigned different users.
	Constraints~(\ref{eq:mx-link-m1}) and~(\ref{eq:mx-link-m2}) are optional, however in our experience they speed up the solution process.
	Finally, constraints~(\ref{eq:mx-link1}) and~(\ref{eq:mx-link2}) link the $M$-variables with the $x$-variables.
	
\emph{Constraint-Satisfaction} (CS) rerpesentation, introduced in~\cite{KarapetyanPGG19}, consists of $k$ integer variables $y_s \in [n]$, $s \in S$, such that $y_s$ is the index of the user assigned to step $s$.
    \added{Unlike the other two representations, the CS representation involves non-Boolean variables and hence it cannot be used with SAT or Pseudo-Boolean solvers; it requires more expressive solvers such as CSP, Satisfiability Modulo Theories or Integer Programming solvers.}
    
	It turns out that a CSP solver, with a certain encoding of constraints, may demonstrate FPT-like performance on instances of WSP with UI constraints~\cite{KarapetyanPGG19}.


\subsection{Encoding of WSP constraints}

In this section, we explain how some common constraints and constraint classes can be encoded in each formulation.

\bigskip

\noindent
\textbf{Binding of duty.}
The binding of duty constraint requests that $s_1$ and $s_2$ are assigned to the same user.
This is a basic UI constraint that can actually be preprocessed; replace steps $s_1$ and $s_2$ with one new step, update the authorisations and other constraints accordingly.
Alternatively, we can encode it as $x_{s_1,u} = x_{s_2,u}$ for every $u \in U$ (in UDPB), or $M_{s_1,s_2} = 1$ (in PBPB), or $y_{s_1} = y_{s_2}$ (in CS).

\bigskip

\noindent
\textbf{Separation of duty.}
The separation of duty constraint requests that $s_1$ and $s_2$ are assigned different users.
This is another basic UI constraint.
We can encode it as $\lnot x_{s_1,u} \lor \lnot x_{s_2,u}$ for every $u \in U$ (in UDPB), or $M_{s_1,s_2} = 0$ (in PBPB), or $y_{s_1} \neq y_{s_2}$ (in CS).

\bigskip

\noindent
\textbf{Counting constraints.}
Counting constraints are more complex UI constraints; they restrict the number of users assigned to the scope steps $T \subseteq S$.
The most studied counting constraints are at-most-$r$ and at-least-$r$ constraints which request that the number of users assigned to any of steps $T$ is at most $r$ or at least $r$, respectively.

With the UDPB solution representation, one can introduce auxiliary Boolean variables $z_u$ for $u \in U$ such that $z_u = 1$ iff user $u$ is assigned at least one step in $T$.
Then it only remains to place a constraint on $\sum_{u \in U} z_u$.

Any UDPB encoding will also work in PBPB, however it is possible to produce more compact formulations for PBPB.
For example, one can count variables $M_{s_1, s_2}$, $s_1, s_2 \in T$, which have values 1.
This approach is usually sufficient for reasonably-sized counting constraints.
Alternatively, one can explicitly enumerate all feasible/infeasible scenarios.
For example, the at-most-$r$ constraints can be encoded by requesting that in every subset $T' \subset T$ of size $r + 1$, there is at least one pair $s_1 \neq s_2 \in T'$ such that $M_{s_1, s_2} = 1$.

The latter approach also works with the CS solution representation, and is the preferred encoding in our experience.

\bigskip

\noindent
\textbf{Other UI constraints.}
With the UDPB and PBPB solution representations, any UI constraint can be expressed via the $x_{s,u}$ variables; however, such an encoding will not fully exploit the special structure of the UI constraints, as it will have to explicitly enumerate all the users.
In this sense, the PBPB solution representation has a significant advantage.
It was shown in~\cite{KarapetyanPGG19} that any UI constraint can be formulated using the $M$-variables only, in which case the encoding is likely to be more compact and separated from the variables concerned with user identities.
The latter point is crucial for achieving the FPT-like performance as it decouples the hard part of the problem (the pattern-related one) from the easy one (the user assignment).

The CS representation also allows encoding of the UI constraints without enumerating all possible values of $y_s$, hence allowing FPT-time solution; indeed, FPT-like performance on the CS-based formulation was observed in~\cite{KarapetyanPGG19}.

\bigskip

\noindent
\textbf{Non-UI constraints.}
As discribed in Section~\ref{sec:CDA}, the non-UI constraints can be absorbed into CDA authorisations.
Here we discuss how CDAs can be encoded using Boolean variables.

Let $\{ \cA^{(c)}_{1}, \cA^{(c)}_{2}, \ldots, \cA^{(c)}_{r} \} = \{ \cA^{(c)}_{p} : p \in P(S) \}$ be a set of distinct $p$-authorisation families in the authorisation family of a constraint $c$, where $r$ is the number of such families.
Let $Q(\cA_p) = \{ p' \in P(S) : \cA_{p'}^{(c)} = \cA_p \}$, i.e.\ $Q(\cA_p)$ is the set of patterns that share the $p$-authorisation family $\cA_p$.
Note that $\{ Q(\cA^{(c)}_{1}), Q(\cA^{(c)}_{2}), \ldots, Q(\cA^{(c)}_{r})\}$ forms a partition of $P(S)$.
Introduce an auxiliary Boolean variable $g_i$ for $i \in [r]$ and add a constraint
\begin{equation}
\bigvee_{i=1}^r g_i = 1
\end{equation}
to ensure that at least one $g_i$ is set to 1.
Now, for $i \in [r]$, enforce that if $g_i = 1$ then $P(\pi) \in Q(\cA^{(c)}_{i})$, where $\pi$ is a solution to the WSP\@.
Enforcing certain patterns is a UI constraint, and hence it can be encoded using the techniques discussed above.
For a compactly formulated constraint, it will usually be possible to encode such an enforcement compactly.

If $g_i = 1$, we need to enforce the $p$-authorisation family $\cA^{(c)}_i$.
We may use the following encoding for every $i \in [r]$.
Let $\cA^{(c)}_{i} = \{ A^{(c)}_{i,1}, A^{(c)}_{i,2}, \ldots, A^{(c)}_{i,t} \}$ for some $t$.
We will need auxiliary Boolean variables $a_{j}$ for $j \in [t]$.
Ensure that at least one of $a_1, a_2, \ldots, a_t$ is set to 1 if $g_i$ is 1:
\begin{equation}
\bigvee_{j = 1}^{t} a_j = g_i.
\end{equation}
For the UDPB and PBPB representations, enforce the authorisation functions by requesting
\begin{equation}
a_j \implies \lnot x_{s,u} \quad \forall s \in S,\ u \notin A^{(c)}_{i,j}(s),\ j \in [t].
\end{equation}
For the CS representation, use
\begin{equation}
a_j \implies y_{s} \neq u \quad
	\forall s \in S,\ u \notin A^{(c)}_{i,j}(s),\ j \in [t].
\end{equation}

In practice, many constraints do not require the full power of CDAs\@.
For example, it may be that $r = 1$ or $t = 1$ for every $i \in [r]$.
In these cases, some of the above steps can be trivially skipped.

\section{Computational experiments}

The aims of our computational study are as follows:
\begin{itemize}
	\item
	Compare the performance of the solvers and formulations.
	
	\item
	Compare computational difficulty of UI and non-UI WSP constraints.
	
	\item
	Study how the solution times change with the change of the problem sizes $k$ and $n$.
	
	\item
	Make the benchmark instances and solvers publicly available to enable further research and possible use.
\end{itemize}

\added{All the experiments were conducted on a machine based on two Xeon E5-2630 v2 CPUs (2.60 GHz), with 32 GB of RAM\@.
Each solver was using exactly one thread and the number of parallel experiments was limited to the number of physical CPU cores.
The instance generator was implemented in C\# and the off-the-shelf solvers were executed from Python code.}

\added{The instance generator, benchmark instances and all the solvers are publicly available at \url{http://doi.org/10.17639/nott.7116}.}

\subsection{Instance generator}
\label{sec:generator}

In a thought experiment, let us define a WSP instance without constraints and then keep adding constraints to it one by one.
The original instance, without the constraints, is likely to have many solutions, and finding one is typically very easy.
Each new constraint will usually reduce the number of solutions and make the instance computationally harder, until some constraint will make the instance unsatisfiable.
From then on, adding more constraints will be making the instance easier again, as the conflicts between the constraints will be more apparent.

The region of the instance space around this boundary between satisfiable and unsatisfiable instances is called \emph{phase transition} (PT), see e.g.~\cite{GentW94,Grant97}\@.
A satisfiable PT instance has only a few valid plans, and adding a constraint is likely to make the instance unsatisfiable.
An unsatisfiable PT instance is likely to become satisfiable if any of its constraints is removed.
PT instances are the hardest and, hence, they are useful for two reasons:
(i)~performance of the solvers on the PT instances is an indicator of the practical worst case;
(ii)~the PT can be used as a reference point, allowing comparison of instances of different structure.

The PT study of WSP was first conducted in~\cite{KarapetyanPGG19} for a generator with UI constraints only.
In this paper, we adopt a similar methodology for the WSP with arbitrary constraints\@.

Our instance generator is based on the instance generator of~\cite{KaGaGu15}.
The authorisation lists are random, with the number of steps varying uniformly between 1 and $\lfloor k / 2 \rfloor$.
The constraints are also produced randomly and uniformly.
The generator supports the following constraints:
\begin{itemize}
	\item
	Separation-of-duty, denoted as SoD;
	
	\item
	At most 3 users, with scope of size 5; denoted as AM3;
	
	\item
	SUAL with scope size 5, $h = 3$ and 5 users in $X$;
	
	\item
	WL with scope size 2, 2 teams, each of size $\lfloor n / 4 \rfloor$;
	
	\item
	ADA with $|U_1| = |U_2| = \lfloor n / 2 \rfloor$.
\end{itemize}
The parameters of the constraints were selected based on the assumptions about real-world applications and also to make them of comparable restrictiveness.

Observe that the SoD and AM3 constraints have in some sense the opposite action, one increasing the diversity of the user assignment whereas the other one decreasing it.
As a result, a mixture of these constraints gives reliable control over satisfiability of instances.\footnote{Observe that, e.g., SoD on their own are unlikely to make a random instance unsatisfiable; even if there is an SoD constraint for every pair of steps and each user is authorised to one step only, the instance is likely to be satisfiable if $n \ge \Theta(k \ln k)$ (see the Coupon Collector's problem~\cite{NewmanS60}).}

Let $\text{WSP}(k, n, m_1 \langle T_1 \rangle, m_2 \langle T_2 \rangle, \dots)$ be the set of WSP instances with $m_1$ constraints of type $\langle T_1 \rangle$, $m_2$ constraints of type $\langle T_2 \rangle$, etc.
We assume that this set of instances is infinite, in the sense that we can draw as many instances as we need, although we allow replacement; drawing an instance from this set corresponds to running the instance generator with a random seed value.
We will say that a set of instances is PT if the frequencies of picking satisfiable and unsatisfiable instances are close to each other.

Let $e_{k,n}$ be the number of SoD constraints such that $\text{WSP}(k, n, k \langle \text{AM3} \rangle, e_{k,n} \langle \text{SoD} \rangle)$ is PT\@.
We use two classes of instances in our experiments:
\begin{enumerate}
	\item
	$\text{WSP}(k, n, k \langle \text{AM3} \rangle, e_{k,n} \langle \text{SoD} \rangle)$ for comparison of solvers.
	As a shortcut, we will use the following notation: $\text{WSP}(k, n)$.
	These instances represent typical problems with UI constraints only.
	
	\item
	$\text{WSP}(k, n, k \langle \text{AM3} \rangle, 0.75 e_{k,n} \langle \text{SoD} \rangle, e^{\langle c \rangle}_{k,n} \langle c \rangle)$ for some constraint type $c$, where $e^{\langle c \rangle}_{k,n}$ is selected to make the instances PT\@.
	As a shortcut, we will use the following notation: $\text{WSP}_{\langle c \rangle}(k, n)$.
	Note that $\text{WSP}_{\langle \text{SoD} \rangle}(k, n) = \text{WSP}(k, n)$.
	If $c$ is non-UI, these instances represent problems with a mixture of UI and non-UI constraints.
\end{enumerate}

\ifdefined\arxiv

\added{The values of $e_{k,n}$ and $e^{\langle c \rangle}_{k,n}$ are reported in Figure~\ref{fig:pt}.
Figure~\ref{fig:pt-10k} provides the values along the $n = 10k$ slice, i.e.\ for instances such that $n = 10k$, while Figure~\ref{fig:pt-18} provides the values along the $k = 18$ slice.}

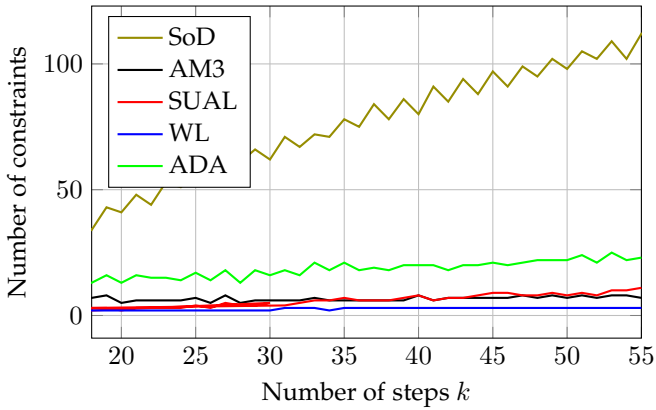
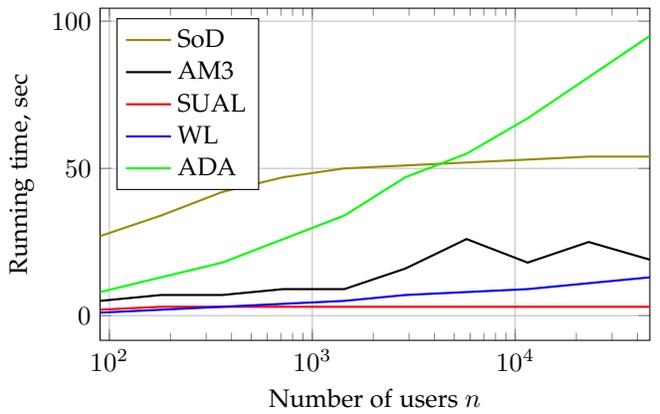
\begin{figure*}[tb]
\subfloat[The values of $e_{k,n}$ and $e^{\langle c \rangle}_{k,n}$ along the $n = 10k$ slice.]
{
	\begin{tikzpicture}
	\begin{axis}[
		compat=newest,
		width=\columnwidth,
		height=6cm,
		legend pos=north west,
		xlabel={Number of steps $k$},
		ylabel={Number of constraints},
		title={},
		grid=major,
		xmin=18,
		xmax=55,
		legend cell align=left,
		every axis plot post/.append style={thick, mark=none}
	]
        \addplot+[olive, solid] table[
	        col sep=tab,
			x=k,
			y=e,
		] {pt-10k-SoD.txt};
		\addlegendentry{SoD}

        \addplot+[black, solid] table[
	        col sep=tab,
			x=k,
			y=e,
		] {pt-10k-AM3.txt};
		\addlegendentry{AM3}

        \addplot+[red, solid] table[
	        col sep=tab,
			x=k,
			y=e,
		] {pt-10k-SUAL.txt};
		\addlegendentry{SUAL}

        \addplot+[blue, solid] table[
	        col sep=tab,
			x=k,
			y=e,
		] {pt-10k-WL.txt};
		\addlegendentry{WL}

        \addplot+[green, solid] table[
	        col sep=tab,
			x=k,
			y=e,
		] {pt-10k-ADA.txt};
		\addlegendentry{ADA}
        \end{axis}		
	\end{tikzpicture}
\label{fig:pt-10k}
}
\subfloat[The values of $e_{k,n}$ and $e^{\langle c \rangle}_{k,n}$ along the $k = 18$ slice.]
{
	\begin{tikzpicture}
	\begin{semilogxaxis}[
		compat=newest,
		width=\columnwidth,
		height=6cm,
		legend pos=north west,
		xlabel={Number of users $n$},
		ylabel={Running time, sec},
		title={},
		grid=major,
		xmin=90,
		xmax=46080,
		legend cell align=left,
		every axis plot post/.append style={thick, mark=none}
	]
	    \addplot+[olive, solid] table[
	        col sep=tab,
			x=n,
			y=e,
		] {pt-18-SoD.txt};
		\addlegendentry{SoD}

        \addplot+[black, solid] table[
	        col sep=tab,
			x=n,
			y=e,
		] {pt-18-AM3.txt};
		\addlegendentry{AM3}
        \addplot+[red, solid] table[
	        col sep=tab,
			x=n,
			y=e,
		] {pt-18-SUAL.txt};
		\addlegendentry{SUAL}
		
		\addplot+[blue, solid] table[
	        col sep=tab,
			x=n,
			y=e,
		] {pt-18-WL.txt};
		\addlegendentry{WL}

        \addplot+[green, solid] table[
	        col sep=tab,
			x=n,
			y=e,
		] {pt-18-ADA.txt};
		\addlegendentry{ADA}

	\end{semilogxaxis}
	\end{tikzpicture}
	\label{fig:pt-18}
}
\caption{The values of $e_{k,n}$ and $e^{\langle c \rangle}_{k,n}$ along the $k = 18$ slice that we obtained to produce PT instances along each slice.
In the legend, SoD stands for $\text{WSP}_{\langle \text{SoD} \rangle}(k, n)$, etc.}
\label{fig:pt}
\end{figure*}

\added{Along the $n = 10k$ slice, the values of $e^{\langle c \rangle}_{k,n}$ grow approximately linearly for all constraint types.
Observe that, by construction, $e^{\langle \text{AM3} \rangle}_{k,n} \approx 0.25k$; the other constraints (SUAL, WL and ADA) seem to have comparable properties.
Along the $k = 18$ slice, however, the number of the SoD constraints converges relatively quickly while the number of the other constraints grows, approximately following the $\log n$ scaling.
While it is difficult to give a detailed explanation of the behaviours of $e^{\langle c \rangle}_{k,n}$, it is easy to observe that the number of SoD constraints in the $k = 18$ slice is trivially bounded by ${k \choose 2} = 153$, hence the growth of $e_{k,n}$ has to slow down.
In fact, the growth of the other graphs is also expected to slow down for similar reasons but not as early.}
\else
\added{For analysis of the $e_{k,n}$ and $e^{\langle c \rangle}_{k,n}$ values in our instances see an extended version of this paper at~\cite{arxiv-version}.}
\fi

\subsection{Comparison of solvers and formulations}
\label{sec:solvers-experiments}

In this experiment we compare the combinations of general-purpose solvers and solution representations.
Specifically, the solvers are: CSP solver `OR-Tools', Satisfiability Modulo Theories solver `Z3' and Pseudo-Boolean solver `SAT4J'\@.
OR-Tools and Z3 support integer variables and, hence, can work with every solution representation.
SAT4J only supports Boolean variables, hence it does not work with the CS solution representation.

The results are reported in Table~\ref{tab:matrix}.
\added{For each combination of solver, solution representation and value of $k$, we solve 100 instances from $\text{WSP}(k, 10k)$.  Then we report the maximum value of $k$ such that the median running time is under 60~seconds.}

\begin{table}[htb]
\begin{center}
\begin{tabular}{llll}
\toprule
	& UDPB
	& PBPB
	& CS \\
\midrule
SAT4J
	& 16
	& 47
	& -- \\
OR-Tools
	& 19
	& 50
	& 41 \\
Z3
	& 10
	& 43
	& 34 \\
\bottomrule
\end{tabular}
\end{center}

\caption{
	The maximum size $k$ of the WSP instances such that the median solution time is under 60~seconds; the larger the better.
}
\label{tab:matrix}
\end{table}

One can see that OR-Tools clearly outperforms Z3 on every formulation.
It also slightly outperforms SAT4J\@.
The best solution representation is PBPB, followed by CS\@.
This is consistent with the findings of~\cite{KarapetyanPGG19}, although the difference between PBPB and CS was less obvious in the previous experiments, partly because the PBPB formulation was only tested with SAT4J, and also because in this study we use a more efficient encoding of the AM3 constraints.

In further experiments, we only use the winning OR-Tools solver, and study the behaviour of the two most efficient formulations, PBPB and CS.

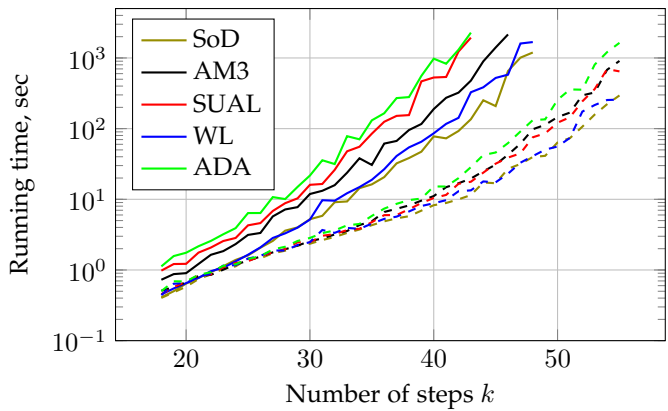
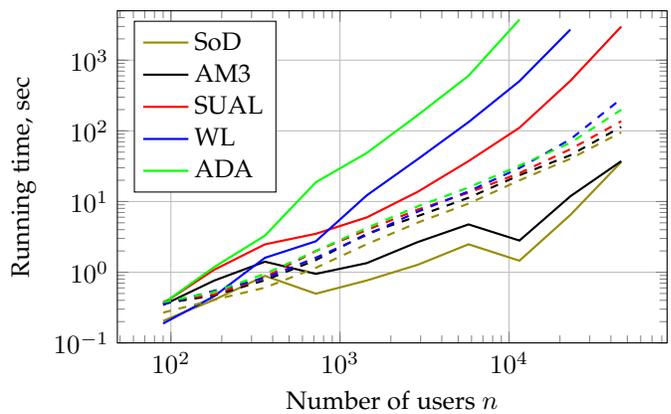
\begin{figure*}[tb]
\subfloat[Scaling of the running times along the slice $n = 10k$.]
{
	\begin{tikzpicture}
	\begin{semilogyaxis}[
		compat=newest,
		width=\columnwidth,
		height=6cm,
		legend pos=north west,
		xlabel={Number of steps $k$},
		ylabel={Running time, sec},
		title={},
		grid=major,
		ymin=0.1,
		ymax=5000,
		legend cell align=left,
		every axis plot post/.append style={thick, mark=none}
	]
        \addplot+[olive, solid] table[
	        col sep=tab,
			x=k,
			y={OrTools CS},
		] {scaling-SoD-10k.txt};
		\addlegendentry{SoD}
		
	    \addplot+[black, solid] table[
	        col sep=tab,
			x=k,
			y={OrTools CS},
		] {scaling-AM3-10k.txt};
		\addlegendentry{AM3}

        \addplot+[red, solid] table[
	        col sep=tab,
			x=k,
			y={OrTools CS},
		] {scaling-SUAL-10k.txt};
		\addlegendentry{SUAL}

        \addplot+[blue, solid] table[
	        col sep=tab,
			x=k,
			y={OrTools CS},
		] {scaling-WL-10k.txt};
		\addlegendentry{WL}

        \addplot+[green, solid] table[
	        col sep=tab,
			x=k,
			y={OrTools CS},
		] {scaling-ADA-10k.txt};
		\addlegendentry{ADA}

		\addplot+[olive, dashed] table[
	        col sep=tab,
			x=k,
			y={OrTools PBPB},
		] {scaling-SoD-10k.txt};	

		\addplot+[black, dashed] table[
	        col sep=tab,
			x=k,
			y={OrTools PBPB},
		] {scaling-AM3-10k.txt};	

        \addplot+[red, dashed] table[
	        col sep=tab,
			x=k,
			y={OrTools PBPB},
		] {scaling-SUAL-10k.txt};
		
		\addplot+[blue, dashed] table[
	        col sep=tab,
			x=k,
			y={OrTools PBPB},
		] {scaling-WL-10k.txt};

        \addplot+[green, dashed] table[
	        col sep=tab,
			x=k,
			y={OrTools PBPB},
		] {scaling-ADA-10k.txt};
		
	\end{semilogyaxis}		
	\end{tikzpicture}
\label{fig:scaling-10k}
}
\subfloat[Scaling of the running times along the slice $k = 18$.]
{
	\begin{tikzpicture}
	\begin{loglogaxis}[
		compat=newest,
		width=\columnwidth,
		height=6cm,
		legend pos=north west,
		xlabel={Number of users $n$},
		ylabel={Running time, sec},
		title={},
		grid=major,
		ymin=0.1,
		ymax=5000,
		legend cell align=left,
		every axis plot post/.append style={thick, mark=none}
	]
	    \addplot+[olive, solid] table[
	        col sep=tab,
			x=n,
			y={OrTools CS},
		] {scaling-SoD-kIs18.txt};
		\addlegendentry{SoD}

        \addplot+[black, solid] table[
	        col sep=tab,
			x=n,
			y={OrTools CS},
		] {scaling-AM3-kIs18.txt};
		\addlegendentry{AM3}
        \addplot+[red, solid] table[
	        col sep=tab,
			x=n,
			y={OrTools CS},
		] {scaling-SUAL-kIs18.txt};
		\addlegendentry{SUAL}
		
		\addplot+[blue, solid] table[
	        col sep=tab,
			x=n,
			y={OrTools CS},
		] {scaling-WL-kIs18.txt};
		\addlegendentry{WL}

        \addplot+[green, solid] table[
	        col sep=tab,
			x=n,
			y={OrTools CS},
		] {scaling-ADA-kIs18.txt};
		\addlegendentry{ADA}

        \addplot+[olive, dashed] table[
	        col sep=tab,
			x=n,
			y={OrTools PBPB},
		] {scaling-SoD-kIs18.txt};
		
        \addplot+[black, dashed] table[
	        col sep=tab,
			x=n,
			y={OrTools PBPB},
		] {scaling-AM3-kIs18.txt};

        \addplot+[red, dashed] table[
	        col sep=tab,
			x=n,
			y={OrTools PBPB},
		] {scaling-SUAL-kIs18.txt};

        \addplot+[blue, dashed] table[
	        col sep=tab,
			x=n,
			y={OrTools PBPB},
		] {scaling-WL-kIs18.txt};
		
        \addplot+[green, dashed] table[
	        col sep=tab,
			x=n,
			y={OrTools PBPB},
		] {scaling-ADA-kIs18.txt};
				
	\end{loglogaxis}
	\end{tikzpicture}
	\label{fig:scaling-18}
}
\caption{Scaling of the running times for two different slices through the $(k, n)$ space.
Instances are solved using OR-Tools.
The solid lines correspond to the CSP formulation; dashed lines correspond to the PBPB formulation.
\added{In the legend, SoD stands for $\text{WSP}_{\langle \text{SoD} \rangle}(k, n)$, etc.}}
\end{figure*}

%
%
%

\subsection{Scaling analysis: $n = 10k$}
\label{sec:scaling-experiments}

In this section, we study scaling of the performance in the $n = 10k$ slice, i.e.\ how the solution time depends on $k$ if $n = 10k$.
This slice was identified as a realistic ratio between $n$ and $k$ in the literature.
In~\cite{KarapetyanPGG19}, it was observed that all the best solution methods demonstrated scaling compatible with the theoretical expectations $O(2^{k \log k})$.
Specifically, the running time closely followed $O(2^{k \log k / 13.2})$, where the 13.2 factor reflected the improvements achieved by heuristics.
In this paper, we study scaling of the running time for WSP with different constraint types.

Figure~\ref{fig:scaling-10k} shows how the PBPB and CS formulations solved with OR-Tools behave when the instance size grows.
\added{Each running time is averaged over 100 instances of the given size.}
Observe that PBPB (dashed lines) demonstrated considerably better performance and slightly better consistency across the constraint types compared to CS (solid lines).
Note that the $y$-axis in this figure is logarithmic, hence any straight line corresponds to exponential time growth.
It appears that the running time of both solvers is slightly super-exponential, consistent with the theoretical expectation $O(2^{k \log k})$.
The complexity seems to be consistent between UI and non-UI constraints demonstrating the effectiveness of CDAs.

\subsection{Scaling analysis: $k = 18$}

The main goal of the analysis in the previous section was to understand how the solution time depends on the size of the instance (for various constraint types); we were scaling all the parameters including $k$, $n$, the number of steps to which each user is authorised and the number of constraints.
In this section, the goal is to verify if the solvers demonstrate FPT-like performance, i.e.\ scale polynomially with $n$.
Thus, we fix the number of steps $k$ and vary the number of users $n$.
Most of the parameters of the instances remain unchanged; however, the increase in the number of users makes the instances more `lose' and we have to compensate that by adjusting the number of some of the constraints, to remain in the PT region (for details, see Section~\ref{sec:generator}).

The results of this experiment are presented in Figure~\ref{fig:scaling-18}.
\added{Each running time is averaged over 100 instances of the given size.}
Here, both axes have logarithmic scale, hence any polynomial appears as a straight line.
The plots for PBPB (dashed lines) are nearly straight, 
with the slope corresponding to a polynomial of degree around 1.
In other words, the solution time grows approximately linearly with the number of users, and this applies to both UI and non-UI constraints.
In fact, the PBPB solution time seems to be almost independent of the constraint type.
In contrast, the CS running times (solid lines) significantly depend on the constraint type.
For the UI constraints, the scaling is sub-linear whereas for the non-UI constraints it is super-linear and possibly slightly non-polynomial.
We conclude that PBPB is a more robust formulation, particularly if the set of constraints is diverse, however there exist scenarios when the CS formulation may outperform PBPB.

\section{Conclusions}
\label{sec:c}
While the previous research on the WSP was mainly focused on the family of UI constraints \added{or relatively restrictive generalisations such as class-independent constraints}, in this paper we showed how to solve WSP with arbitrary (reasonable) constraints in FPT-like time.
To achieve this, we generalised the concept of authorisations by making them context-dependent and showed how to absorb non-UI constraints into context-dependent authorisations. 
This allowed us to extend methods developed for WSP with UI constraints to arbitrary constraints.

Previously, it was shown that WSP with UI constraints can be efficiently solved by general-purpose solvers with the appropriate formulations.
We supported and extended this observation in several ways: (i) we formalised the three WSP solution representations as well as techniques to encode various WSP constraints, including non-UI constraints; (ii) we compared the solvers and formulations in a systematic way; (iii) we tested the best solvers/formulations with a range of UI and non-UI constraints to demonstrate that the appropriate formulations are efficient on WSP with arbitrary constraints.
In other words, we demonstrated that the practical complexity of WSP remains FPT-like even if some of the constraints are non-UI.

More specifically, we demonstrated that the running times of solvers with the appropriate WSP formulations are FPT-like.\footnote{\added{Note that experimental research inherently provides only limited evidence.  E.g., in our case, we considered only a small number of types of UI and non-UI constraints.}}
We hope that this result will motivate researchers to look for FPT-aware formulations of other FPT problems, to enable development of practical solution methods with FPT-like running times.

	

Finally, we made the new instance generator, benchmark instances and all the solvers publicly available, to support future studies: \url{http://doi.org/10.17639/nott.7116}.


\bigskip

\noindent
\textbf{Acknowledgement.}
Gregory Gutin's research was supported by the Leverhulme Trust award RPG-2018-161.

\bibliographystyle{IEEEtran}
\bibliography{refs}

\begin{thebibliography}{10}
\providecommand{\url}[1]{#1}
\csname url@samestyle\endcsname
\providecommand{\newblock}{\relax}
\providecommand{\bibinfo}[2]{#2}
\providecommand{\BIBentrySTDinterwordspacing}{\spaceskip=0pt\relax}
\providecommand{\BIBentryALTinterwordstretchfactor}{4}
\providecommand{\BIBentryALTinterwordspacing}{\spaceskip=\fontdimen2\font plus
\BIBentryALTinterwordstretchfactor\fontdimen3\font minus
  \fontdimen4\font\relax}
\providecommand{\BIBforeignlanguage}[2]{{%
\expandafter\ifx\csname l@#1\endcsname\relax
\typeout{** WARNING: IEEEtran.bst: No hyphenation pattern has been}%
\typeout{** loaded for the language `#1'. Using the pattern for}%
\typeout{** the default language instead.}%
\else
\language=\csname l@#1\endcsname
\fi
#2}}
\providecommand{\BIBdecl}{\relax}
\BIBdecl

\bibitem{GutinK20}
G.~Z. Gutin and D.~Karapetyan, ``Constraint branching in workflow
  satisfiability problem,'' in \emph{Proceedings of the 25th {ACM} Symposium on
  Access Control Models and Technologies, {SACMAT} 2020}, J.~Lobo, S.~D.
  Stoller, and P.~Liu, Eds.\hskip 1em plus 0.5em minus 0.4em\relax {ACM}, 2020,
  pp. 93--103.

\bibitem{Cr05}
J.~Crampton, ``A reference monitor for workflow systems with constrained task
  execution,'' in \emph{10th {ACM} Symposium on Access Control Models and
  Technologies, {SACMAT} 2005, Stockholm, Sweden, June 1-3, 2005, Proceedings},
  E.~Ferrari and G.~Ahn, Eds.\hskip 1em plus 0.5em minus 0.4em\relax {ACM},
  2005, pp. 38--47.

\bibitem{ansi-rbac04}
\emph{{ANSI INCITS} 359-2004 for Role Based Access Control}, American National
  Standards Institute, 2004.

\bibitem{BeFeAt99}
E.~Bertino, E.~Ferrari, and V.~Atluri, ``The specification and enforcement of
  authorization constraints in workflow management systems,'' \emph{{ACM}
  Trans. Inf. Syst. Secur.}, vol.~2, no.~1, pp. 65--104, 1999.

\bibitem{BertolissiSR18}
C.~Bertolissi, D.~R. dos Santos, and S.~Ranise, ``Solving multi-objective
  workflow satisfiability problems with optimization modulo theories
  techniques,'' in \emph{Proceedings of the 23nd {ACM} on Symposium on Access
  Control Models and Technologies, {SACMAT} 2018, Indianapolis, IN, USA, June
  13-15, 2018}, 2018, pp. 117--128.

\bibitem{CrGuDM}
J.~Crampton, G.~Z. Gutin, and D.~Majumdar, ``Bounded and approximate strong
  satisfiability in workflows,'' in \emph{Proceedings of the 24th {ACM}
  Symposium on Access Control Models and Technologies, {SACMAT} 2019}, 2019,
  pp. 179--184.

\bibitem{CrGuWa16}
J.~Crampton, G.~Z. Gutin, M.~Kouteck{\'{y}}, and R.~Watrigant, ``Parameterized
  resiliency problems,'' \emph{Theor. Comput. Sci.}, vol. 795, pp. 478--491,
  2019.

\bibitem{dosSantosR17}
D.~R. dos Santos and S.~Ranise, ``On run-time enforcement of authorization
  constraints in security-sensitive workflows,'' in \emph{Software Engineering
  and Formal Methods - 15th International Conference, {SEFM} 2017,
  Proceedings}, 2017, pp. 203--218.

\bibitem{BertolissiSR15}
C.~Bertolissi, D.~R. dos Santos, and S.~Ranise, ``Automated synthesis of
  run-time monitors to enforce authorization policies in business processes,''
  in \emph{Proceedings of the 10th {ACM} Symposium on Information, Computer and
  Communications Security, {ASIA} {CCS} '15}, 2015, pp. 297--308.

\bibitem{CompagnaSPR16}
L.~Compagna, D.~R. dos Santos, S.~E. Ponta, and S.~Ranise, ``Cerberus:
  Automated synthesis of enforcement mechanisms for security-sensitive business
  processes,'' in \emph{Tools and Algorithms for the Construction and Analysis
  of Systems - 22nd International Conference, {TACAS} 2016, Proceedings}, 2016,
  pp. 567--572.

\bibitem{dosSantosPR16}
D.~R. dos Santos, S.~E. Ponta, and S.~Ranise, ``Modular synthesis of
  enforcement mechanisms for the workflow satisfiability problem: Scalability
  and reusability,'' in \emph{Proceedings of the 21st {ACM} on Symposium on
  Access Control Models and Technologies, {SACMAT} 2016}, 2016, pp. 89--99.

\bibitem{WaLi10}
Q.~Wang and N.~Li, ``Satisfiability and resiliency in workflow authorization
  systems,'' \emph{{ACM} Trans. Inf. Syst. Secur.}, vol.~13, no.~4, p.~40,
  2010.

\bibitem{CrGuYe13}
J.~Crampton, G.~Gutin, and A.~Yeo, ``On the parameterized complexity and
  kernelization of the workflow satisfiability problem,'' \emph{{ACM} Trans.
  Inf. Syst. Secur.}, vol.~16, no.~1, pp. 4:1--4:31, 2013.

\bibitem{CoCrGaGuJo14}
D.~Cohen, J.~Crampton, A.~Gagarin, G.~Gutin, and M.~Jones, ``Iterative plan
  construction for the workflow satisfiability problem,'' \emph{J. Artif.
  Intell. Res.}, vol.~51, pp. 555--577, 2014.

\bibitem{CohenCGGJ16}
D.~A. Cohen, J.~Crampton, A.~Gagarin, G.~Gutin, and M.~Jones, ``Algorithms for
  the workflow satisfiability problem engineered for counting constraints,''
  \emph{J. Comb. Optim.}, vol.~32, pp. 3--24, 2016.

\bibitem{KarapetyanPGG19}
D.~Karapetyan, A.~J. Parkes, G.~Z. Gutin, and A.~Gagarin, ``Pattern-based
  approach to the workflow satisfiability problem with user-independent
  constraints,'' \emph{J. Artif. Intel. Res.}, vol.~66, 2019.

\bibitem{BerreP10}
D.~L. Berre and A.~Parrain, ``The sat4j library, release 2.2,'' \emph{{Journal
  on Satisfiability, Boolean Modeling and Computation}}, vol.~7, no. 2-3, pp.
  59--6, 2010.

\bibitem{CrGaGuJoWa16}
J.~Crampton, A.~Gagarin, G.~Gutin, M.~Jones, and M.~Wahlstr{\"{o}}m, ``On the
  workflow satisfiability problem with class-independent constraints for
  hierarchical organizations,'' \emph{{ACM} Trans. Priv. Secur.}, vol.~19,
  no.~3, pp. 8:1--8:29, 2016.

\bibitem{dosSantosR17survey}
D.~R. dos Santos and S.~Ranise, ``A survey on workflow satisfiability,
  resiliency, and related problems,'' \emph{CoRR}, vol. abs/1706.07205, 2017.

\bibitem{HoldererAM15}
J.~Holderer, R.~Accorsi, and G.~M{\"{u}}ller, ``When four-eyes become too much:
  a survey on the interplay of authorization constraints and workflow
  resilience,'' in \emph{Proceedings of the 30th Annual {ACM} Symposium on
  Applied Computing}, 2015, pp. 1245--1248.

\bibitem{CyganFKLMPPS15}
M.~Cygan, F.~V. Fomin, L.~Kowalik, D.~Lokshtanov, D.~Marx, M.~Pilipczuk,
  M.~Pilipczuk, and S.~Saurabh, \emph{Parameterized Algorithms}.\hskip 1em plus
  0.5em minus 0.4em\relax Springer, 2015.

\bibitem{diestel}
R.~Diestel, \emph{Graph Theory, 4th Edition}, ser. Graduate texts in
  mathematics.\hskip 1em plus 0.5em minus 0.4em\relax Springer, 2012, vol. 173.

\bibitem{GuWa16}
G.~Gutin and M.~Wahlstr{\"{o}}m, ``Tight lower bounds for the workflow
  satisfiability problem based on the strong exponential time hypothesis,''
  \emph{Inf. Process. Lett.}, vol. 116, no.~3, pp. 223--226, 2016.

\bibitem{CrGuKa15b}
J.~Crampton, G.~Gutin, and D.~Karapetyan, ``Valued workflow satisfiability
  problem,'' in \emph{Proceedings of the 20th SACMAT}, 2015, pp. 3--13.

\bibitem{ImPa01}
R.~Impagliazzo and R.~Paturi, ``On the complexity of k-{SAT},'' \emph{J.
  Comput. Syst. Sci.}, vol.~62, no.~2, pp. 367--375, 2001.

\bibitem{ImPaZa01}
R.~Impagliazzo, R.~Paturi, and F.~Zane, ``Which problems have strongly
  exponential complexity?'' \emph{J. Comput. Syst. Sci.}, vol.~63, no.~4, pp.
  512--530, 2001.

\bibitem{GentW94}
I.~P. Gent and T.~Walsh, ``The {SAT} phase transition,'' in \emph{Proceedings
  of the Eleventh European Conference on Artificial Intelligence, Amsterdam,
  The Netherlands, August 8-12, 1994}, A.~G. Cohn, Ed.\hskip 1em plus 0.5em
  minus 0.4em\relax John Wiley and Sons, Chichester, 1994, pp. 105--109.

\bibitem{Grant97}
S.~A. Grant, ``Phase transition behaviour in constraint satisfaction
  problems,'' Ph.D. dissertation, University of Leeds, {UK}, 1997.

\bibitem{KaGaGu15}
D.~Karapetyan, A.~V. Gagarin, and G.~Gutin, ``Pattern backtracking algorithm
  for the workflow satisfiability problem with user-independent constraints,''
  in \emph{Frontiers in Algorithmics - 9th International Workshop, {FAW} 2015},
  pp. 138--149.

\bibitem{NewmanS60}
D.~J. Newman and L.~Shepp, ``The double dixie cup problem,'' \emph{American
  Mathematical Monthly}, vol.~67, no.~1, pp. 58--61, 1960.

\end{thebibliography}

\vfill

\begin{IEEEbiographynophoto}{Daniel Karapetyan}
Dr Daniel Karapetyan received his PhD degree in computer science from Royal Holloway, University of London, UK, in 2010.
Since 2020, he has been assistant professor of computer science at the University of Nottingham.  
His research interests include algorithm design, artificial intelligence and data science applied to combinatorial optimisation and decision support.
\end{IEEEbiographynophoto}

\begin{IEEEbiographynophoto}{Gregory Gutin}
Gregory Gutin received the PhD degree in mathematics from Tel Aviv
University, Israel, in 1993. Since 2000 he has been a professor of
computer science with Royal Holloway, University of London, United
Kingdom. His research interests include access control, combinatorial
optimization, graph theory and applications, and parameterized
algorithms and complexity.
\end{IEEEbiographynophoto}
\end{document}